\documentclass[lettersize,journal]{IEEEtran}
\usepackage{amsmath,amsfonts}
\usepackage{algorithmic}
\usepackage{algorithm}
\usepackage{array}
\usepackage[caption=false,font=normalsize,labelfont=sf,textfont=sf]{subfig}
\usepackage{textcomp}
\usepackage{stfloats}
\usepackage{url}
\usepackage{verbatim}
\usepackage{graphicx}
\usepackage{cite}

\usepackage{microtype}
\usepackage{booktabs}
\usepackage{hyperref}
\usepackage{amssymb}
\usepackage{mathtools}
\usepackage{amsthm}
\usepackage{enumitem}

\usepackage{bm}
\usepackage{color}
\usepackage{diagbox}
\usepackage{lineno}
\usepackage{multirow}
\usepackage{float}
\usepackage{units}

\theoremstyle{plain}
\newtheorem{theorem}{Theorem}[section]
\newtheorem{proposition}[theorem]{Proposition}
\newtheorem{lemma}[theorem]{Lemma}

\theoremstyle{definition}
\newtheorem{definition}[theorem]{Definition}

\newtheorem{problem}[theorem]{Problem}
\theoremstyle{remark}
\newtheorem{remark}[theorem]{Remark}

\def\1{\mathbf{1}}
\def\0{\mathbf{0}}

\def\x{{\bf x}}

\def\y{{\bf y}}

\def\S{{\bf S}}

\def\tr{{\mathrm{tr}}}
\def\nnz{{\mathrm{nnz}}}

\DeclarePairedDelimiter{\abs}{\lvert}{\rvert}
\DeclarePairedDelimiter{\norm}{\lVert}{\rVert}
\DeclarePairedDelimiter{\ceil}{\lceil}{\rceil}
\DeclarePairedDelimiter{\flr}{\lfloor}{\rfloor}
\DeclarePairedDelimiter{\prn}{\lparen}{\rparen}
\DeclarePairedDelimiter{\brk}{\lbrack}{\rbrack}

\newcommand*\dif{\mathop{}\!\mathrm{d}}

\begin{document}

\title{Optimal Randomized Approximations for Matrix-based R\'enyi's Entropy}

\author{
    Yuxin~Dong,
    Tieliang~Gong,
    Shujian~Yu,
    Chen~Li
\thanks{Corresponding author: Tieliang Gong}
\thanks{Y. Dong(dongyuxin@stu.xjtu.edu.cn), T. Gong (adidasgtl@gmail.com), C. Li (Cli@xjtu.edu.cn) and are with the School of Computer Science and Technology, and Key Laboratory of Intelligent Networks and Network Security, Ministry of Education, Xi'an 710049, China}
\thanks{S. Yu (yusj9011@gmail.com) is with the Department of UiT The Arctic University of Norway}
}

\markboth{}%
{Dong \MakeLowercase{\textit{et al.}}: Optimal Randomized Approximations for Matrix-based R\'enyi's Entropy}


\maketitle

\begin{abstract}
    The Matrix-based R\'enyi's entropy enables us to directly measure information quantities from given data without the costly probability density estimation of underlying distributions, thus has been widely adopted in numerous statistical learning and inference tasks. However, exactly calculating this new information quantity requires access to the eigenspectrum of a semi-positive definite (SPD) matrix $A$ which grows linearly with the number of samples $n$, resulting in a $O(n^3)$ time complexity that is prohibitive for large-scale applications. To address this issue, this paper takes advantage of stochastic trace approximations for matrix-based R\'enyi's entropy with arbitrary $\alpha \in \mathbb{R}^+$ orders, lowering the complexity by converting the entropy approximation to a matrix-vector multiplication problem. Specifically, we develop random approximations for integer-order $\alpha$ cases and polynomial series approximations (Taylor and Chebyshev) for non-integer $\alpha$ cases, leading to a $O(n^2sm)$ overall time complexity, where $s, m \ll n$ denote the number of vector queries and the polynomial order respectively. We theoretically establish statistical guarantees for all approximation algorithms and give explicit order of $s$ and $m$ with respect to the approximation error $\epsilon$, showing optimal convergence rate for both parameters up to a logarithmic factor. Large-scale simulations and real-world applications validate the effectiveness of the developed approximations, demonstrating remarkable speedup with negligible loss in accuracy.
\end{abstract}


\section{Introduction}
\IEEEPARstart{T}{he} R\'enyi's $\alpha$-order entropy, introduced by Alfred R\'enyi \cite{renyi1961measures}, serves as a one-parameter generalization of the well-known Shannon's entropy. Following R\'enyi's work, extensive studies have been conducted in machine learning and statistical inference tasks, demonstrating elegant properties and impressive scalability \cite{principe2010information, teixeira2012conditional, giraldo2014measures, yu2021measuring}. However, its heavy dependence on the underlying data distributions makes estimation of high-dimensional probability density functions (PDF) inevitable, which is especially expensive or even intractable due to the curse of high-dimensionality \cite{fan2006statistical}.

Recently, the matrix-based R\'enyi's entropy \cite{giraldo2014measures, giraldo2014information} is introduced as a substitution that can be quantified directly from given data samples. Inspired by the quantum generalization of R\'enyi's definition \cite{lennert2013quantum}, this new family of information measures is defined on the eigenspectrum of a normalized Hermitian matrix constructed by projecting data points in reproducing kernel Hilbert space (RKHS), thus avoiding explicit estimation of underlying data distributions. Because of its intriguing property in high-dimensional scenarios, the matrix-based R\'enyi's entropy and mutual information have been successfully applied in various data science applications, ranging from classical dimensionality reduction \cite{brockmeier2017quantifying, alvarez2017kernel} and feature selection \cite{yu2019multivariate} problems to advanced deep learning problems such as robust learning against covariant shift \cite{yu2021measuring}, network pruning \cite{sarvani2021hrel} and knowledge distillation \cite{miles2021information}.

Nevertheless, calculating this new information measure requires complete knowledge about the eigenspectrum of a Gram matrix, whose size grows linearly with the number of samples $n$, resulting in a $O(n^3)$ time complexity with traditional eigenvalue algorithms including eigenvalue decomposition, singular value decomposition, CUR decomposition and QR factorization \cite{mahoney2009cur, watkins2008qr}, greatly hampering its practical application on large-scale datasets.

To address this issue, we develop efficient approximations for matrix-based R\'enyi's entropy from the perspective of randomized numerical linear algebra. Motivated by the recent advancement of a variance reduced stochastic trace estimator named Hutchinson++ (Hutch++) \cite{meyer2021hutchpp}, we decompose the kernel matrix $A$ by randomly projecting it into an orthogonal subspace which holds the largest eigenvalues with high probability, and the counterpart that holds smaller eigenvalues. Their traces are then exactly calculated and approximated by the original Hutchinson algorithm respectively, leading to an optimal $O(1/\epsilon)$ convergence rate in terms of the number of vector queries. We further develop polynomial expansion techniques including Taylor and Chebyshev series to approximate arbitrary matrix power functional in R\'enyi's entropy. We theoretically analyze the quality-of-approximation results and conduct large-scale experiments to validate the effectiveness of this framework. Our main contributions in this work are summarized as follows:
\begin{itemize}[topsep=0pt]
	\setlength\itemsep{0pt}
	\item We develop efficient approximations for matrix-based R\'enyi's entropy with randomized trace estimation and polynomial approximation techniques. Our algorithms reduce the overall time complexity from $O(n^3)$ to $O(n^2sm)$ ($s, m \ll n$) and support arbitrary $\alpha$ values.
	\item We theoretically establish both upper and lower bounds for approximation accuracy, showing that the convergence rates $O(1/\epsilon)$ and $O(\sqrt{\kappa})$ ($\kappa$ is the condition number of $A$) w.r.t $s$ and $m$ respectively are nearly optimal up to a logarithmic factor in terms of approximation error.
	\item We evaluate our algorithms on large-scale simulation datasets and real-world information-related tasks, demonstrating promising speedup with only negligible loss in validation accuracy.
\end{itemize}

\section{Preliminaries}
The R\'enyi's $\alpha$-order entropy $H_\alpha(X)$ is defined on the PDF $p(\x)$ for a given continuous random variable $X$ that values in a finite set $\mathcal{X}$:
\begin{equation}
	H_{\alpha}(X) = \frac{1}{1 - \alpha} \log \int_{\mathcal{X}} p^\alpha(\x) \dif \x,
\end{equation}
where the limit case $\alpha \rightarrow 1$ yields Shannon's entropy. It is easy to see that calculating R\'enyi's entropy requires knowledge about data distributions, which hampers its application in high-dimensional scenarios. To solve this issue, Giraldo et al. proposed an alternative entropy measure that enables direct quantification from given data:
\begin{definition} \cite{giraldo2014measures}
	\label{th:renyi}
	Let $\kappa: \mathcal{X} \times \mathcal{X} \mapsto \mathbb{R}$ be a real valued positive kernel that is also infinitely divisible \cite{bhatia2006infinitely}. Given $ \{\x_i \}_{i=1}^n \subset \mathcal{X}$, each $\x_i$ being a real-valued scalar or vector, and the Gram matrix $K$ obtained from $K_{ij} = \kappa(\x_i, \x_j)$, a matrix-based analogue to R\'enyi's $\alpha$-entropy can be defined as:
	\begin{equation*}
		S_\alpha(A) = \frac{1}{1-\alpha}\log(\tr(A^\alpha)) = \frac{1}{1-\alpha}\log\brk*{\sum_{i=1}^n \lambda_i^\alpha(A)},
	\end{equation*}
	where $A_{ij} = \frac{1}{n}\frac{K_{ij}}{\sqrt{K_{ii}K_{jj}}}$ is a normalized kernel matrix and $\lambda_i(A)$ denotes the $i$-th eigenvalue of $A$.
\end{definition}
The normalized kernel matrix $A$ is symmetric semi-positive definite (SPD) with unit trace, therefore its eigenvalues are in $[0, 1]$ and satisfies $\sum_{i=1}^n \lambda_i(A) = \tr(A) = 1$. We denote the minimum and maximum eigenvalue of $A$ as $u \in [0, 1/n]$ and $v \in [1/n, 1]$ respectively, and the corresponding condition number is then $\kappa = v/u$. In numerical scenarios, the power iteration and Lanczos iteration are effective algorithms for calculating $u$ and $v$ in $O(d \cdot \nnz(A))$, where $\nnz(\cdot)$ denotes the number of non-zero elements in a matrix and $d$ is the number of iterations.

\begin{definition} \cite{yu2019multivariate}
	\label{th:renyi_joint}
	Let $\kappa_1: \mathcal{X}^1 \times \mathcal{X}^1 \mapsto \mathbb{R}$, $\cdots$, $\kappa_L: \mathcal{X}^L \times \mathcal{X}^L \mapsto \mathbb{R}$ be positive infinitely divisible kernels and $\{\x_i^1, \cdots, \x_i^L\}_{i=1}^n \subset \mathcal{X}^1 \times \cdots \times \mathcal{X}^L$ be a collection of $n$ samples, a matrix-based analogue to R\'enyi's $\alpha$-order joint entropy among $L$ variables can be defined as:
	\begin{equation}
		S_\alpha(A_1, \cdots, A_L) = S_\alpha\prn*{ \frac{A_1 \circ \cdots \circ A_L}{\tr(A_1 \circ \cdots \circ A_L)}}, \label{eq:joint_entropy}
	\end{equation}
	where $A_1, \cdots, A_L$ are normalized kernel matrices and $\circ$ denotes the Hadamard product.
\end{definition}
Within these settings, the matrix-based R\'enyi's $\alpha$-order conditional entropy $S_\alpha(A_1, \cdots, A_k | B)$ and mutual information $I_\alpha(\{A_1, \cdots, A_k\}; B)$ between variables $\x^1, \cdots, \x^k$ and $\y$ can be defined as:
\begin{align}
	S_\alpha(A_1, \cdots, A_k | B) &= S_\alpha\prn*{A_1, \cdots, A_k, B} - S_\alpha(B), \label{eq:cond_entropy} \\
	I_\alpha(\{A_1, \cdots, A_k\}; B) &= S_\alpha\prn*{A_1, \cdots, A_k} \nonumber \\
	&\quad - S_\alpha(A_1, \cdots, A_k | B), \label{eq:mutual_info}
\end{align}
where $A_1, \cdots, A_k$ and $B$ are corresponding kernel matrices constructed from $\x^1, \cdots, \x^k$ and $\y$. As we can see, the matrix-based R\'enyi's entropy functionals above avoid estimation of underlying data distributions, which makes them easily applicable in high-dimensional scenarios. Moreover, it is simple to verify that they are permutation invariant to the ordering of variables $A_1, \cdots, A_k$.

\section{Approximation Algorithms}
In this section, we develop efficient approximations for matrix-based R\'enyi's entropy from the perspective of randomized numerical linear algebra. Inspired by Hutch++ \cite{meyer2021hutchpp}, the recently developed randomized trace estimator that achieves promising performance and strong statistical guarantees, we design following efficient trace approximation algorithm for arbitrary positive matrix function $f(A)$ as shown in Algorithm \ref{alg:hutchpp}. By decomposing the kernel matrix $A$ into a randomized orthogonal subspace $Q$ and its complement $I - Q Q^\top$, we achieve nearly optimal convergence rate in terms of the number of vector queries $s$.

We first establish the connection between trace estimation and matrix-based R\'enyi's entropy approximation:
\begin{proposition}
	\label{th:trace_upper_lower}
	For any $\epsilon \in (0, 1)$ and sufficient large $n$, if a randomized algorithm $\mathcal{A}$ can estimate the trace of any $n \times n$ SPD matrix $A$ to relative error $1 \pm \epsilon$ with success probability at least $1 - \delta$ using $s$ queries, then $\mathcal{A}$ can be used to estimate $S_\alpha(A) = \frac{1}{1-\alpha}\log\tr(A^\alpha)$ to relative error $1 \pm \epsilon_0$ with the same success probability using $s$ queries where $\epsilon = 1-\min(\mu,1/\mu)^{\epsilon_0}$ and
	\begin{equation*}
	    \mu = \frac{1 - un}{v - u} \cdot v^\alpha + \frac{vn - 1}{v - u} \cdot u^\alpha.
	\end{equation*}
	Vice versa for $\epsilon = \max(n^{\alpha-1},n^{1-\alpha})^{\epsilon_0}-1$.
\end{proposition}
Proposition \ref{th:trace_upper_lower} implies that the trace estimation problem is equivalent to matrix-based R\'enyi's entropy approximation. By taking $f(A) = A^\alpha$, Algorithm \ref{alg:hutchpp} generate a $(1 \pm \epsilon)$ approximation for $S_\alpha(A)$ with high probability in $O(s\cdot\nnz(A))$, which is substantially lower than traditional $O(n^3)$ eigenvalue based approaches.

\subsection{Integer Order Approach}
\begin{algorithm}[tb]
	\caption{Hutch++ algorithm for implicit matrix trace estimation \cite{meyer2021hutchpp}}
	\label{alg:hutchpp}
	\begin{algorithmic}[1]
		\STATE \textbf{Input:} Kernel matrix $A \in \mathbb{R}^{n \times n}$, number of random vectors $s (s \ll n)$, positive matrix function $f(A)$.
		\STATE \textbf{Output:} Approximation to $\tr(f(A))$.
		\STATE Sample $S \in \mathbb{R}^{n \times \frac{s}{4}}, G \in \mathbb{R}^{n \times \frac{s}{2}}$ from i.i.d. standard Gaussian distribution.
		\STATE Compute an orthonormal basis $Q \in \mathbb{R}^{n \times \frac{s}{4}}$ for the span of $AS$ via QR decomposition.
		\STATE \textbf{Return:} $Z = \tr\prn*{Q^\top f(A)Q} +$ \\
		\qquad $\frac{2}{s} \tr\prn*{G^\top(I-QQ^\top)f(A)(I-QQ^\top)G}$.
	\end{algorithmic}
\end{algorithm}

When $\alpha \in \mathbb{N}$, for any real-valued vector $g$, $A^\alpha \cdot g$ could be directly calculated by multiplying $A$ with a vector for $\alpha$ times. This observation gives Algorithm \ref{alg:int} for integer order R\'enyi's entropy estimation:
\begin{theorem}
	\label{th:int_upper}
	Let $\tilde{S}_\alpha(A)$ be the output of Algorithm \ref{alg:int} with $s = O\prn*{\frac{1}{\epsilon}\sqrt{\log\prn*{\frac{1}{\delta}}} + \log\prn*{\frac{1}{\delta}}}$, then with probability at least $1 - \delta$:
	\begin{equation*}
		\abs*{\tilde{S}_\alpha(A) - S_\alpha(A)} \le \epsilon \cdot S_\alpha(A).
	\end{equation*}
\end{theorem}
\begin{remark}
	Theorem \ref{th:int_upper} establishes the main quality-of-approximation result for Algorithm \ref{alg:int}, that a $s$ with order $O(1/\epsilon)$ is sufficient to guarantee the approximation error with high probability. Algorithm \ref{alg:int} finishes in $O(\alpha s \cdot \nnz(A))$, which is substantially lower than eigenvalue decomposition algorithms.
\end{remark}

\subsection{Taylor Series Approach}
The fractional-order of $\alpha$ may constantly be come across in real-world applications \cite{yu2019multivariate} depending on the specific tasks. In this circumstance, obtaining an exact value of $A^\alpha \cdot g$ is not feasible for random vector $g$. An ideal workaround is to adopt a Taylor expansion on the power term $A^\alpha$:
\begin{equation*}
	(1+x)^\alpha = \sum_{k=0}^\infty \binom{\alpha}{k}x^k, ~x \in [-1, 1]
\end{equation*}
Taking $v$ as the largest eigenvalue of $A$, eigenvalues of $A/v - I_n$ are in $[-1, 0]$. Then $A^\alpha$ can be expanded as:
\begin{equation*}
	A^\alpha = v^\alpha \sum_{k=0}^\infty \binom{\alpha}{k} \prn*{A/v - I_n}^k.
\end{equation*}
An approximation to $A^\alpha \cdot g$ is now available by calculating $A \cdot g$, $A^2 \cdot g$, $\cdots$ in sequence. By selecting the first $m$ major terms in the polynomial expansion above, we have Algorithm \ref{alg:taylor} for non-integer order R\'enyi's entropy estimation:

\begin{algorithm}[tb]
	\caption{Integer order matrix-based R\'enyi's entropy estimation}
	\label{alg:int}
	\begin{algorithmic}[1]
		\STATE \textbf{Input:} Kernel matrix $A \in \mathbb{R}^{n \times n}$, number of random vectors $s$, integer order $\alpha \ge 2$.
		\STATE \textbf{Output:} Approximation to $S_\alpha(A)$.
		\STATE Run Hutch++ with $f(A) = A^\alpha$ and $s$ random vectors.
		\STATE \textbf{Return:} $\tilde{\S}_\alpha(A) = \frac{1}{1-\alpha} \log(\text{Hutch++}(A^\alpha))$.
	\end{algorithmic}
\end{algorithm}

\begin{theorem}
	\label{th:taylor_upper}
	Let $\tilde{S}_\alpha(A)$ be the output of Algorithm \ref{alg:taylor} with
	\begin{equation*}
		\begin{aligned}
			s &= O\prn*{\textstyle \frac{1}{\epsilon\abs{\alpha-1}}\sqrt{\log\prn*{\frac{1}{\delta}}} + \log\prn*{\frac{1}{\delta}}}, \\
			m &= O\prn*{\textstyle \kappa\log\prn*{\frac{1}{\epsilon\abs{\alpha-1}}}},
		\end{aligned}
	\end{equation*}
	where $\kappa = v/u$ is the condition number of $A$, then for any normalized kernel matrix $A$ with eigenvalues in [u, v], with probability at least $1 - \delta$:
	\begin{equation*}
		\abs*{\tilde{S}_\alpha(A) - S_\alpha(A)} \le \epsilon \cdot S_\alpha(A).
	\end{equation*}
\end{theorem}
\begin{remark}
	Theorem \ref{th:taylor_upper} presents the relative error bound for Algorithm \ref{alg:taylor}, where explicit order of $s$ and $m$ are given to guarantee the approximation accuracy. Specifically, $s$ is scaled by a coefficient $1/\abs{\alpha-1}$ compared to Theorem \ref{th:int_upper}, and $m$ is positively related to the condition number $\kappa$. Algorithm \ref{alg:taylor} finishes in $O(ms \cdot \nnz(A))$ with $m, s \ll n$.
\end{remark}

The analysis above requires $u > 0$, i.e. the kernel matrix has full rank. However, this requirement is hard to be satisfied in some machine learning tasks e.g. RKHS transporting and dimension reduction \cite{zhang2019optimal, harandi2014manifold}, where rank deficient matrices are frequently encountered. To account for this, we establish the following theorem:
\begin{theorem}
	\label{th:taylor_upper_0}
	Let $\tilde{S}_\alpha(A)$ be the output of Algorithm \ref{alg:taylor} with
	\begin{equation*}
		\begin{aligned}
			s &= O\prn*{\textstyle \frac{1}{\epsilon\abs{\alpha-1}}\sqrt{\log\prn*{\frac{1}{\delta}}} + \log\prn*{\frac{1}{\delta}}}, \\
			m &= O\prn*{\textstyle (vn)^\frac{1}{\min(1, \alpha)} \sqrt[\alpha]{\frac{1}{\epsilon\abs{\alpha-1}}}},
		\end{aligned}
	\end{equation*}
	then for any normalized kernel matrix $A$ with eigenvalues in [0, v], with probability at least $1 - \delta$:
	\begin{equation*}
		\abs*{\tilde{S}_\alpha(A) - S_\alpha(A)} \le \epsilon \cdot S_\alpha(A).
	\end{equation*}
\end{theorem}
\begin{remark}
	When $u = 0$, due to the existance of a singular point in $f(x) = x^\alpha$ at $x = 0$, a logarithmic convergence rate is no longer achievable. The polynomial approximation error is now dominated by $\epsilon$ instead of $\kappa$. The coefficient $vn$ corresponds to the rare worst case when the eigenvalues of $A$ all equal $1/n$, or are all in $\{0, v\}$.
\end{remark}

\begin{algorithm}[tb]
	\caption{Non-integer order matrix-based R\'enyi's entropy estimation via Taylor series}
	\label{alg:taylor}
	\begin{algorithmic}[1]
		\STATE \textbf{Input:} Kernel matrix $A \in \mathbb{R}^{n \times n}$, number of random vectors $s$, non-integer order $\alpha$, polynomial order $m$, eigenvalue upper bound $v$.
		\STATE \textbf{Output:} Approximation to $S_\alpha(A)$.
		\STATE Run Hutch++ with $f(A) = v^\alpha \sum_{k=0}^m \binom{\alpha}{k} \prn*{A/v - I_n}^k$ and $s$ random vectors.
		\STATE \textbf{Return:} $\tilde{\S}_\alpha(A) = \frac{1}{1-\alpha} \log(\text{Hutch++}(A^\alpha))$.
	\end{algorithmic}
\end{algorithm}

\subsection{Chebyshev Series Approach}
Chebyshev expansion is an advanced technique to approximate analytic functions and often enjoy better theoretical properties. For some continuous function $f: [-1, 1] \rightarrow \mathbb{R}$, it is defined as
\begin{equation*}
	f(x) = \frac{c_0}{2} + \sum_{k=1}^\infty c_k T_k(x), ~x \in [-1,1]
\end{equation*}
where $T_{k+1}(x) = 2xT_k(x) - T_{k-1}(x)$ for $k \ge 1$, $T_0(x) = 1$ and $T_1(x) = x$. By taking the first $m$ terms, the coefficients $c_k, k = 0, \cdots, m$ could be calculated as
\begin{equation*}
	c_k = \frac{2}{m+1} \sum_{i=0}^{m} f(x_i)T_k(x_i),
\end{equation*}
where $x_i = \cos\big(\pi(i+1/2)/(m+1)\big)$. Through a combination with linear mapping $g$: $[-1,1] \rightarrow [u,v]$, we can now approximate $f(\lambda) = \lambda^\alpha$ for any $\lambda \in [u,v]$ with $\hat{T}_k = T_k \circ g^{-1}$, $k = 0, \cdots, m$, as shown in Algorithm \ref{alg:chebyshev}.
\begin{algorithm}[b]
	\caption{Non-integer order matrix-based R\'enyi's entropy estimation via Chebyshev series}
	\label{alg:chebyshev}
	\begin{algorithmic}[1]
		\STATE \textbf{Input:} Kernel matrix $A \in \mathbb{R}^{n \times n}$, number of random vectors $s$, non-integer order $\alpha$, polynomial order $m$, eigenvalue lower \& upper bounds $u$, $v$.
		\STATE \textbf{Output:} Approximation to $S_\alpha(A)$.
		\STATE Set $v \leftarrow \max(v, u + 2\sqrt{2u-u^2})$.
		\STATE Run Hutch++ with $f(A) = c_0/2 + \sum_{k=1}^m c_k\hat{T}_k(A)$ and $s$ random vectors.
		\STATE \textbf{Return:} $\tilde{\S}_\alpha(A) = \frac{1}{1-\alpha} \log(\text{Hutch++}(A^\alpha))$.
	\end{algorithmic}
\end{algorithm}

\begin{theorem}
	\label{th:cheby_upper}
	Let $\tilde{S}_\alpha(A)$ be the output of Algorithm \ref{alg:chebyshev} with
	\begin{equation*}
		\begin{aligned}
			s &= O\prn*{\textstyle \frac{1}{\epsilon\abs{\alpha-1}}\sqrt{\log\prn*{\frac{1}{\delta}}} + \log\prn*{\frac{1}{\delta}}}, \\
			m &= O\prn*{\textstyle \sqrt{\kappa}\log\prn*{\frac{\kappa}{\epsilon\abs{\alpha-1}}}},
		\end{aligned}
	\end{equation*}
	where $\kappa = v/u$ is the condition number of $A$, then for any normalized kernel matrix $A$ with eigenvalues in [u, v], with probability at least $1 - \delta$:
	\begin{equation*}
		\abs*{\tilde{S}_\alpha(A) - S_\alpha(A)} \le \epsilon \cdot S_\alpha(A).
	\end{equation*}
\end{theorem}
\begin{remark}
	Theorem \ref{th:cheby_upper} requires only $O(\sqrt{\kappa})$ polynomial terms to guarantee the approximation accuracy for Algorithm \ref{alg:chebyshev} in the case that A is well-conditional, comparing to Theorem \ref{th:taylor_upper} which require $O(\kappa)$ to achieve the same approximation accuracy. Moreover, Algorithm \ref{alg:chebyshev} requires estimation of $u$, which is generally more difficult than estimating $v$ because of its small magnitude.
\end{remark}

Similarly, we establish the error bound of Chebyshev series for rank deficient kernel matrices.
\begin{theorem}
	\label{th:cheby_upper_0}
	Let $\tilde{S}_\alpha(A)$ be output of Algorithm \ref{alg:chebyshev} with
	\begin{equation*}
		\begin{aligned}
			s &= O\prn*{\textstyle \frac{1}{\epsilon\abs{\alpha-1}}\sqrt{\log\prn*{\frac{1}{\delta}}} + \log\prn*{\frac{1}{\delta}}}, \\
			m &= O\prn*{\textstyle (vn)^\frac{1}{2\min(1, \alpha)} \sqrt[2\alpha]{\frac{1}{\epsilon\abs{\alpha-1}}}},
		\end{aligned}
	\end{equation*}
	then for any normalized kernel matrix $A$ with eigenvalues in [0, v], with probability at least $1 - \delta$:
	\begin{equation*}
		\abs*{\tilde{S}_\alpha(A) - S_\alpha(A)} \le \epsilon \cdot S_\alpha(A).
	\end{equation*}
\end{theorem}
\begin{remark}
	Compared with Theorem \ref{th:taylor_upper_0}, Algorithm \ref{alg:chebyshev} within rank deficient case still achieves better theoretical guarantees from all perspectives.
\end{remark}

\subsection{Connection with the Lanczos Method}
Besides polynomial approximation, an alternative approach for approximating matrix functions is the Lanczos method \cite{bellalij2015bounding}: given implicit matrix $f(A)$ and arbitrary vector $b$, an approximation of $f(A) \cdot b$ is acquired by a linear interpolation in the Krylov subspace $\{b, A \cdot b, \cdots, A^m \cdot b\}$. This could be interpreted as an adaptive polynomial approximation technique, where the coefficients are chosen according to the given matrix $f(A)$ and vector $b$. However, this approach does not achieve any faster convergence rate than explicit polynomial approximation: as pointed out in \cite{ubaru2017fast}, the block Lanczos method achieves exactly the same upper bound $O(\kappa \log(\kappa/\epsilon))$ as Chebyshev series in terms of subspace dimension, while requiring additional $O(nms)$ memory to store the block vectors in each step \cite{bellalij2015bounding}. Moreover, the lower bound of the Lanczos method is also closely related to the lower bound of polynomial approximation \cite{musco2018stability}, which is discussed in the next section. We leave this for future research.

\section{Lower Bounds}
So far, we established estimation algorithms for matrix-based R\'enyi's entropy and evaluated their theoretical properties. A natural question is if the $O(1/\epsilon)$, $O(\sqrt{\kappa})$ or $O(\sqrt[2\alpha]{1/\epsilon})$ upper bounds in our previous analysis are tight. In this section, we will prove that up to a logarithmic factor, they are consistent with theoretical lower bounds.

In Proposition \ref{th:trace_upper_lower}, we show that an effective trace approximator implies an effective approximator for matrix-based Renyi's entropy. Based on the lower bound of randomized implicit trace estimation in fixed precision model \cite{meyer2021hutchpp}, we obtain the lower bound of required matrix-vector multiplication queries $s$ by complexity reduction:
\begin{theorem}
	\label{th:entropy_lower}
	Any algorithm that accesses a normalized $n \times n$ kernel matrix $A$ via matrix-vector multiplication queries $Ar_1, \cdots, Ar_m$, where $r_1, \cdots, r_m$ are possibly adaptively chosen randomized vectors under limited precision computation model, requires $s = \Omega\prn*{\frac{1}{\epsilon\abs{\alpha-1}\log n\log\prn*{1/\epsilon\abs{\alpha-1}\log n}}}$ such queries to output an estimate $Z$ so that, with probability at least $\frac{2}{3}$, $\abs*{Z - S_\alpha(A)} \le \epsilon \cdot S_\alpha(A)$ for arbitrary $\alpha > 0$.
\end{theorem}
\begin{remark}
	Note that the lower bound $s = \Omega(1/\epsilon)$ matches our previous results up to a $\log(1/\epsilon)$ factor, which means our error bounds are nearly-optimal. Moreover, the scaling term $1/(1-\alpha)$ implies that precise approximation is impossible when $\alpha \rightarrow 1$. This observation is confirmed in our simulation studies (Section \ref{exp:simulation}). 
\end{remark}

Next, by applying the theory of best uniform approximation error, we establish the lower bounds for the required number of terms $m$ in polynomial approximation. Given a continuous real function $f$ defined on $[-1,1]$, denote the $m$-terms best uniform approximation of $f$ by $p_m$, then:
\begin{equation*}
	\norm{f - p_m} = \min_{p \in \mathbb{P}_m}\norm{f - p},
\end{equation*}
where $\norm{\cdot}$ denotes the uniform norm and $\mathbb{P}_m$ is the linear space of all polynomials with degree at most $m$. Based on previous theoretical analysis of function $f(x) = x^\alpha$ \cite{lam1972some, bernsteincollected, varga1992some, bernstein1938meilleure}, we obtain:
\begin{theorem}
	\label{th:poly_lower}
	There exists a positive decreasing function $\epsilon_0: \mathbb{R}^+ \rightarrow \mathbb{R}^+$ such that for arbitrary $0 < u < v < 1$ and $0 < \epsilon < \epsilon_0(v/u)$, any polynomial $p_m$ that approximates matrix function $f(A)=A^\alpha$, requires $m = \Omega\prn*{\sqrt{\kappa} \log\prn*{\frac{1}{\kappa\epsilon\abs{\alpha-1}\log n}}}$ degree to achieve
	\begin{equation*}
		\abs*{\frac{1}{1-\alpha}\log\big(\tr\prn*{p_m(A)}\big) - S_\alpha(A)} \le \epsilon \cdot S_\alpha(A),
	\end{equation*}
	for any positive definite matrix $A$ with all eigenvalues in $[u, v]$ and $\tr(A) \in [1, 2]$, where $\kappa = v/u$.
\end{theorem}
\begin{theorem}
	\label{th:poly_lower_0}
	For arbitrary $v > 0$ and small enough $\epsilon$, any polynomial $p_m$ that approximates matrix function $f(A)=A^\alpha$, requires $m = \Omega\prn*{\sqrt[2\alpha]{\frac{1}{\epsilon\abs{\alpha-1}\log n}}}$ degree to achieve
	\begin{equation*}
		\abs*{\frac{1}{1-\alpha}\log\big(\tr\prn*{p_m(A)}\big) - S_\alpha(A)} \le \epsilon \cdot S_\alpha(A),
	\end{equation*}
	for any positive semi-definite matrix $A$ with all eigenvalues in $[0, v]$ and $\tr(A) \in [1, 2]$.
\end{theorem}
\begin{remark}
	Theorem \ref{th:poly_lower} and \ref{th:poly_lower_0} present the lower bound for polynomial approximation in non-integer order R\'enyi's entropy estimation. Also, these bounds indicate the near-optimality of Algorithm \ref{alg:chebyshev} in consideration of the results in Theorem \ref{th:cheby_upper} and \ref{th:cheby_upper_0}.
\end{remark}

\section{Experimental Results}
In this section, we evaluate the performance of proposed approximations implemented in C++ using Eigen \cite{eigenweb}. Numerical studies are conducted on an Intel i7-10700 (2.90GHz) CPU with 64GB of RAM, with deep learning models trained on an RTX 2080Ti GPU. We give comprehensive experimental results for both synthetic data and real-world information-related tasks.

\begin{figure}[tb]
	\centering
	\includegraphics[width=0.45\textwidth]{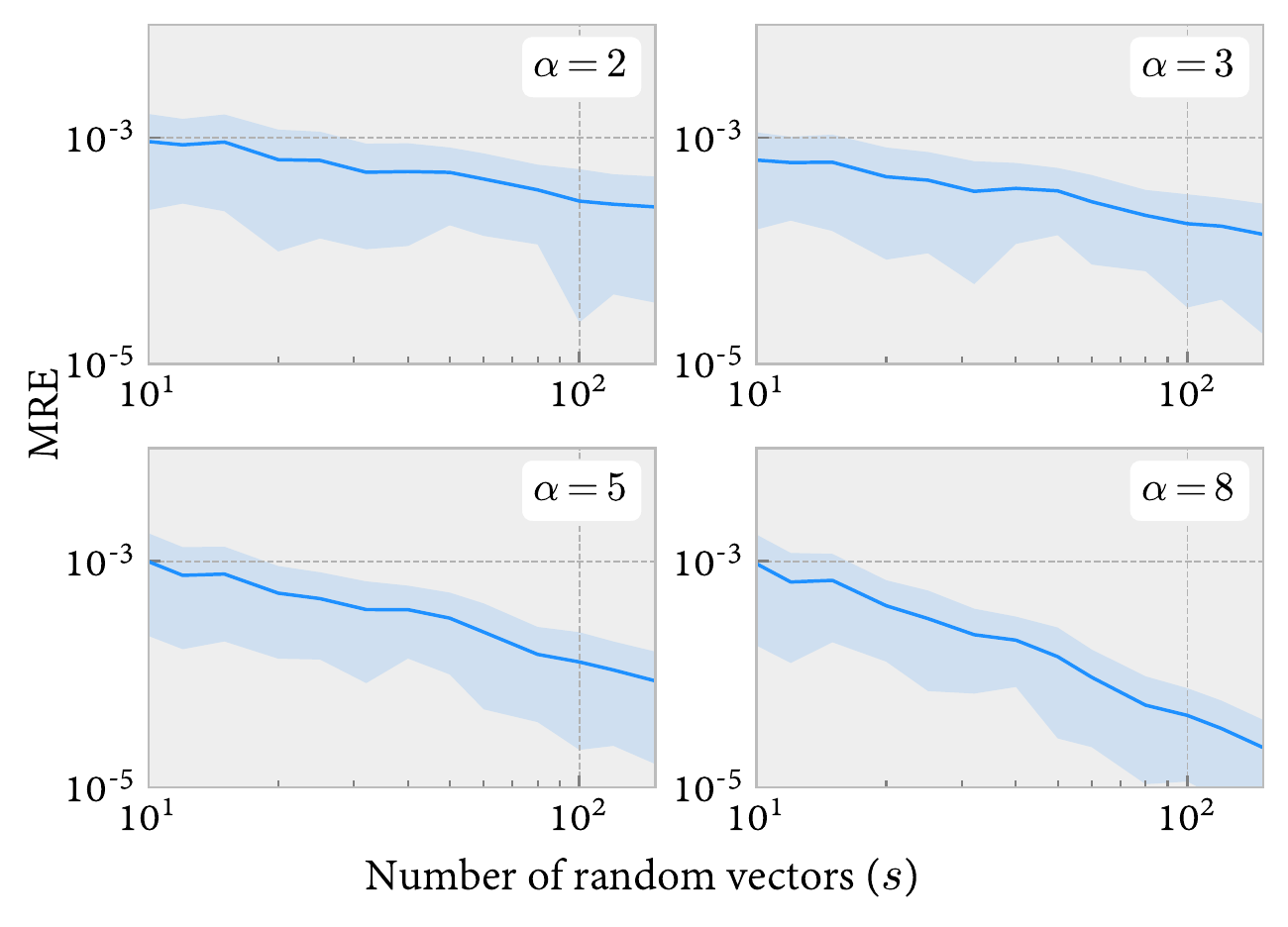}
	\caption{Number of random vectors $s$ versus MRE curves for integer $\alpha$-order R\'enyi's entropy estimation.}
	\label{IntExp}
\end{figure}
\begin{figure}[t]
	\centering
	\includegraphics[width=0.45\textwidth]{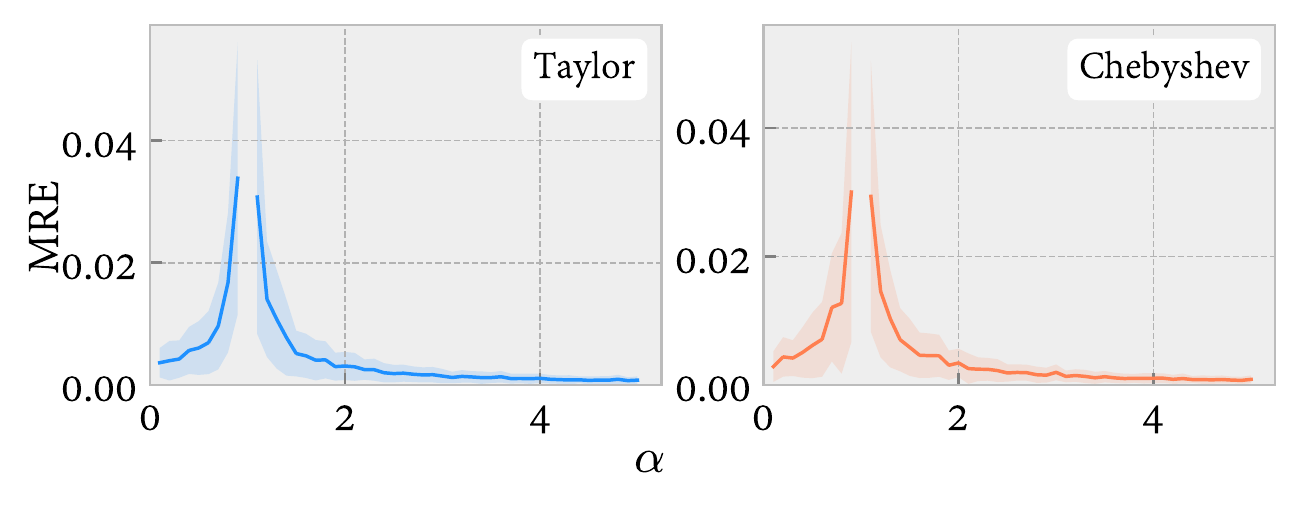}
	\caption{$\alpha$ versus MRE curves for non-integer $\alpha$-order R\'enyi's entropy estimation algorithms.}
	\label{AlphaExp}
\end{figure}
\begin{figure}[t]
	\centering
	\includegraphics[width=0.45\textwidth]{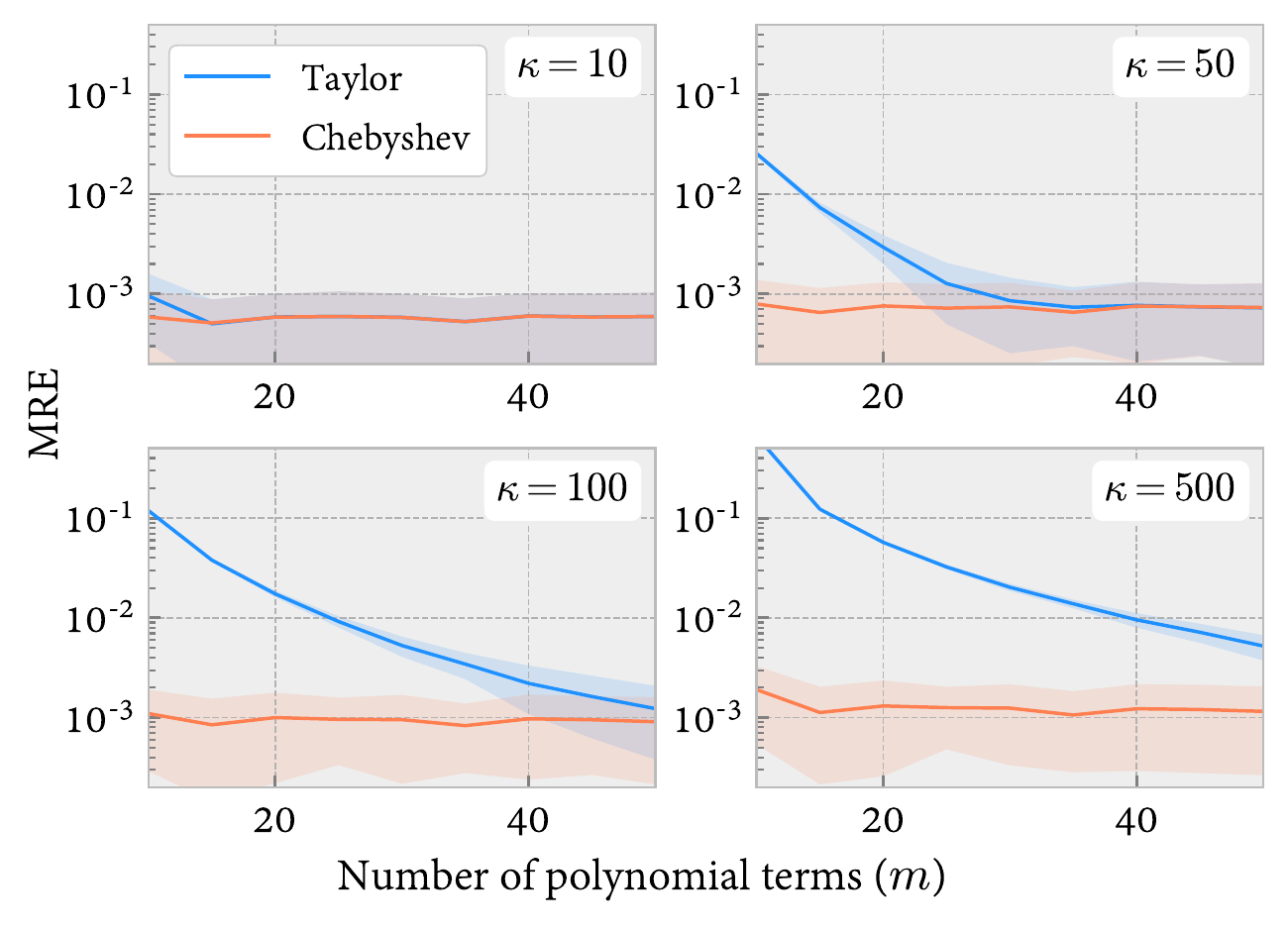}
	\caption{Number of polynomial terms $m$ versus MRE curves for different condition numbers $\kappa$.}
	\label{PolyExp}
\end{figure}
\begin{figure}[t]
	\centering
	\includegraphics[width=0.45\textwidth]{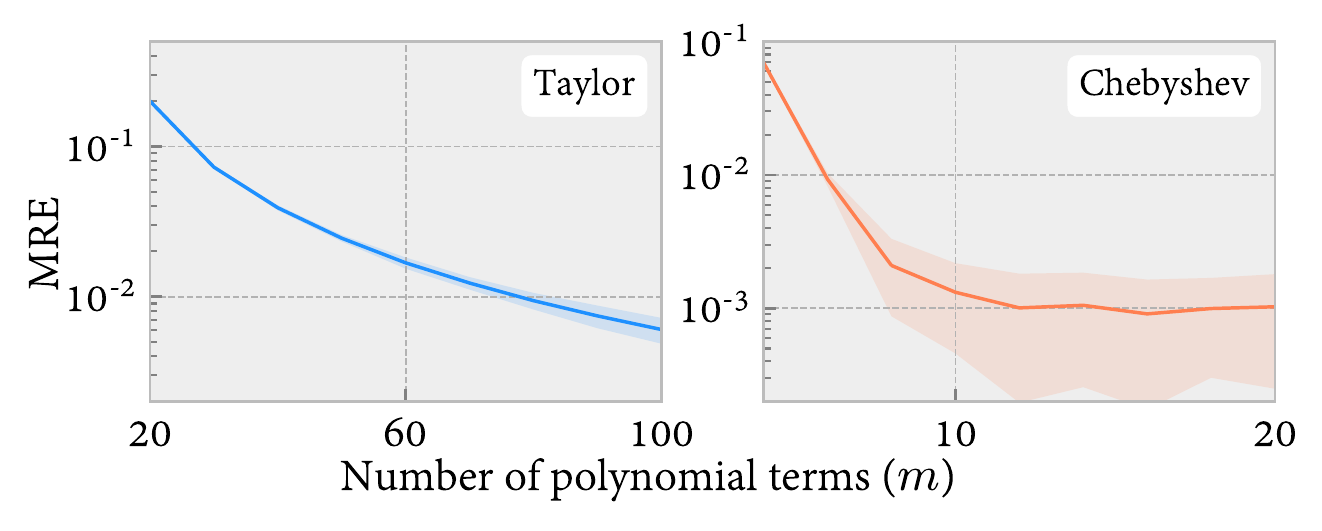}
	\caption{Number of polynomial terms $m$ versus MRE curves for rank deficient kernel matrices.}
	\label{DeficientExp}
\end{figure}

\subsection{Simulation Studies} \label{exp:simulation}
In the following simulation experiments, we generate synthetic data points by mixture of Gaussian distribution $\frac{1}{2}N(-1,I_d)+\frac{1}{2}N(1,I_d)$ with $n = 5,000$ and $d = 10$, where $I_d$ is an identity matrix of size $d$, resulting in a $5,000 \times 5,000$ kernel matrix size. Gaussian kernel $\kappa(\x_i, \x_j) = \exp(-\| \x_i - \x_j \|_2^2/2\sigma^2)$ with $\sigma = 1$ is adopted in matrix-based R\'enyi's entropy quantification. For each benchmark, we report the mean relative error (MRE) and corresponding standard deviation (SD) of approximation results after $K = 100$ trials. The oracle $S_\alpha(A)$ is computed through the trivial $O(n^3)$ eigenvalue approach.

\subsubsection{Integer Order Approximation}
We first evaluate the performance of Algorithm \ref{alg:int} for integer-order entropy estimation. We report the $s$ versus MRE curves for $\alpha \in \{2, 3, 5, 8\}$, where the number of random vectors $s$ ranges from $10$ to $150$, as shown in Figure \ref{IntExp}. The shaded area indicates the corresponding SD of MRE. We observe a linear relationship between $s$ and MRE as expected. It is worth noting that we achieve a $0.1\%$ relative error with only $s = 10$ random vectors, which costs roughly $1.2$ seconds of running time for $\alpha=2$. For comparison, the trivial eigenvalue approach takes $27$ seconds to obtain a complete eigenvalue decomposition.

\subsubsection{Non-Integer Order Approximation}
We further evaluate the Taylor and Chebyshev algorithms for non-integer $\alpha$ orders. The results on describing the impact of $\alpha$ on approximation MRE with $m=20$ and $s=100$ are reported in Figure \ref{AlphaExp}. As expected, MRE curves grow with the increase of $\alpha$ for $\alpha <1$ and decrease otherwise. This phenomenon is because of the $|\frac{1}{1-\alpha}|$ coefficient in our previous theoretical analysis. When $\alpha$ is close to $1$, this term dominates the approximation error.

\begin{figure*}[t]
	\centering
	\includegraphics[width=0.9\textwidth]{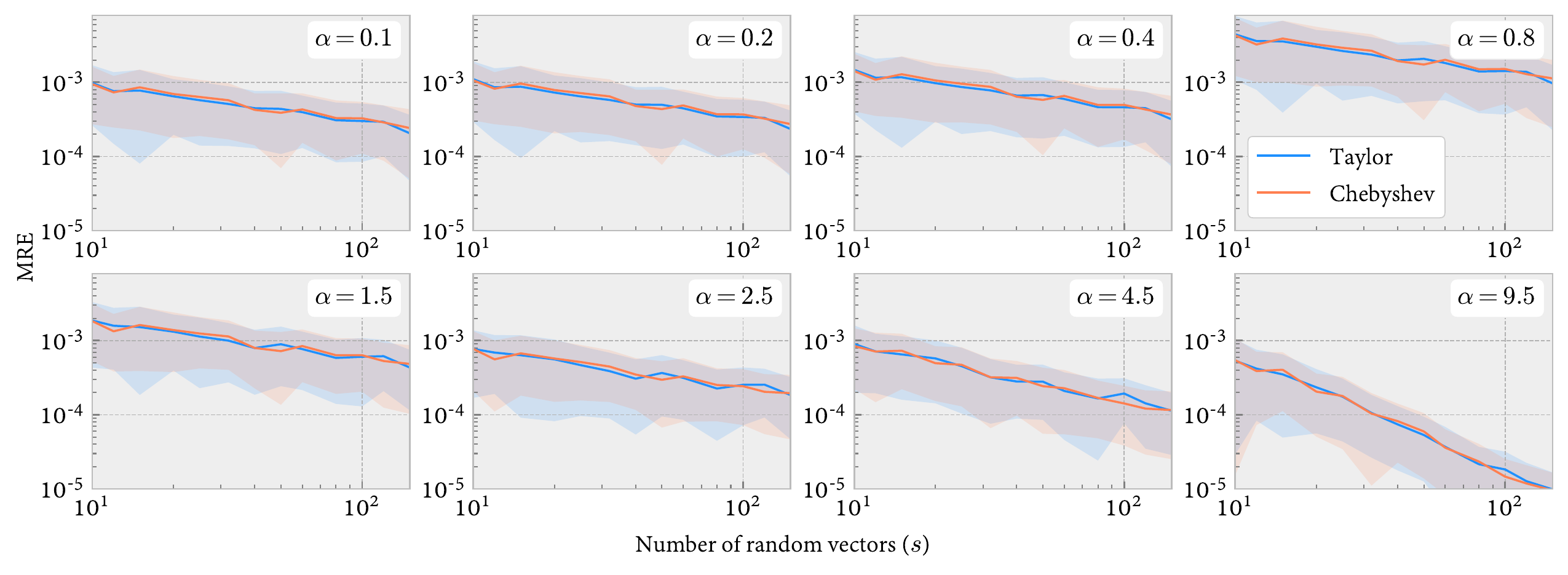}
	\caption{Number of random vectors $s$ versus MRE curves for non-integer $\alpha$-order R\'enyi's entropy estimation.}
	\label{NonIntExp}
\end{figure*}

We next explore the influence of different condition numbers $\kappa$ in polynomial approximation. Here we set $\alpha=1.5$, $s=100$ and $m$ ranges from $10$ to $50$, with adjusted width parameter $\sigma$ in Gaussian kernel to control the eigenspectrum. It can be seen from Figure \ref{PolyExp} that the polynomial terms $m$ required by Taylor approximation is larger than that of Chebyshev approximation for relatively large $\kappa$, which verifies our findings in Theorem \ref{th:taylor_upper} and \ref{th:cheby_upper}. For smaller $\kappa$, the two approaches yield comparable results. Reminding that the Taylor series does not require estimation of $u$, it is thus more suitable for kernel matrices with flat eigenspectrum.

We further investigate the rank deficient circumstances. We adopt the polynomial kernel $\kappa(\x_i, \x_j) = (\x_i^\top \x_j + r)^p$ with $p = 2$, $r = 1$ to fulfill the infinitely divisible requirement, and set $d = 98$ to retain roughly $1\%$ of the eigenvalues to be zero. We find that Chebyshev approximation still outperforms the Taylor approach in terms of MRE with small $m$ values, as shown in Figure \ref{DeficientExp}.

Finally, we report the experimental results of both algorithms for different $\alpha$ values. We set $\sigma = 1$ in Gaussian kernel and $s$ varies from $10$ to $150$. For the Taylor approach, we set $m = 40$ and for Chebyshev we set $m = 15$. From Figure \ref{NonIntExp}, we can see that the two approaches achieve similar performance in this setting. Also, we get relatively higher MRE for $\alpha$ near $1$ ($\alpha=0.8$), same as we have discussed before. In this sense, we recommend a combination of $s = 50$ and $m = 15$ that takes $3$ seconds to achieve a $10^{-3}$ relative error for most circumstances, leading to $9$ times speedup compared with the trivial eigenvalue approaches. For larger kernel matrices, this advantage could be even more pronounced.

\subsection{Real Data Studies}
In real-world data-driven applications, the extended entropy measures including R\'enyi's $\alpha$-order joint entropy (\ref{eq:joint_entropy}), conditional entropy (\ref{eq:cond_entropy}) and mutual information (\ref{eq:mutual_info}) enable much wider adoption of information-based machine learning tasks. By approximating the trace of the joint kernel matrix $A_1 \circ \cdots \circ A_L$ in (\ref{eq:joint_entropy}), our approximations algorithms are immediately applicable on these extended information measures. In this section, we will demonstrate the performance of our algorithms on these novel extensions by three representative real-world applications, which accelerate the computation of entropy (robust deep learning), mutual information (feature ranking) and multivariate mutual information (feature selection) respectively. We select $\sigma=1$ in Gaussian kernel and $\alpha=2$ for simplicity.

\subsubsection{Application to Robust Deep Learning}
\begin{table}[tb]
	\centering
	\caption{Test error and time spent for different methods on CIFAR-10. Right is the total time of network training, while left is the time spent solely on calculating IB. The number quoted indicates corresponding $\alpha$ value in R\'enyi's entropy.}
	\label{IBRes}
	\begin{sc}
		\begin{tabular}{ llcc }
			\toprule
			Backbone & Objective & Error (\%) & Time (hour) \\
			\midrule
			\multirow{5}{*}{VGG16} & CE & 7.36 & \quad - \enspace / 2.27 \\
			& VIB & 7.15 & 0.03 / 2.30 \\
			& DIB ($1.01$) & 5.66 & \multirow{2}{*}{1.13 / 3.40} \\
			& DIB ($2$) & 5.69 & \\
			\cmidrule{2-4}
			& ADIB ($2$) & 5.71 & 0.15 / 2.42 \\
			\bottomrule
		\end{tabular}
	\end{sc}
\end{table}

The Information Bottleneck (IB) objective was firstly introduced by \cite{tishby1999information} and has recently been adopted in deep network training to learn either stochastic or deterministic compressed yet meaningful representations \cite{yu2021deep, ardizzone2020training, wu2020graph}. Denoting $X$ as the input and $Y$ as the target label, the IB approach learns an intermediate representation $T$ that balances the trade-off between the predictive performance of $T$ on task $Y$ (quantified by $I(Y;T)$) and the complexity of $T$ (quantified by $I(X;T)$):
\begin{equation}
	\mathcal{L}_{IB} = I(Y;T) - \beta \cdot I(X; T), \label{eq:info_bottle}
\end{equation}
where $\beta$ is a hyper-parameter that balances $I(Y;T)$ and $I(X;T)$. There are different ways to parameterize IB by neural networks. In general, the maximization of $I(Y;T)$ is equivalent to the minimization of cross-entropy (CE) loss \cite{alemi2017deep,amjad2019learning}, which turns the objective of deep IB into a standard CE loss regularized by a differentiable mutual information term $I(X;T)$. On the other hand, for a deterministic and feed-forward neural network, we have $I(X;T) = H(T)$, the entropy of latent representation $T$ \cite{amjad2019learning,saxe2018information}.
Hence, (\ref{eq:info_bottle}) could be simply implemented in deep neural networks in an end-to-end style with an additional loss function, which could be acquired by approximating the trace of the kernel matrix constructed from $T$ and directly optimized by gradient-based methods.

We follow the experiment settings used by \cite{yu2021deep}, where VGG16 \cite{simonyan2015very} and CIFAR-10 are selected as the backbone network and classification dataset respectively. The last fully-connected layer in VGG before the softmax layer is selected as the bottleneck $T$. All models are trained for $400$ epochs, with $0.1$ initial learning rate which is reduced by a factor of $10$ every $100$ epochs, and $100$ batch size so the kernel matrices in R\'enyi's entropy is $100 \times 100$. The performance of Deterministic IB (DIB, (\ref{eq:info_bottle})) and Approximated DIB (ADIB) are evaluated with number of random vectors $s=10$ and $\beta = 0.01$. The final classification accuracy and time spent on calculating IB / training networks are reported in Table \ref{IBRes}.

We compare with the Variational IB (VIB) \cite{alemi2017deep}, which achieves state-of-the-art performance in previous studies. For a fair comparison, we further report the results for $\alpha=1.01$ as it is recommended by \cite{yu2021deep}. It can be seen that $\alpha=2$ also works and produces comparable classification accuracy. The results indicate that with our approximation algorithms, we can benefit from the outstanding performance of IB with nearly no increase in training time, while the original approach requires $50\%$ more time for IB calculation. This speedup could be even enlarged by using a larger batch size, which is recommended by modern fine-tuning techniques \cite{samuel2018dont}.

\subsubsection{Application to Feature Ranking}
Given a set of features $S = \{X_1, \cdots, X_n\}$, the feature selection task aims to find the smallest subset $S_{sub}$ that maximize the relevance about the labels $Y$. This ultimate target is to maximize the multivariate mutual information $I_\alpha(S_{sub};Y)$, which is however usually impractical in real-world scenarios because of the curse of dimensionality and the difficulty of global optimization.

\begin{table}[tb]
	\centering
	\caption{Number of instances (\#I), features (\#F), and classes (\#C) of classification datasets used in feature selection and ranking experiments, and the corresponding running time of RMI and ARMI.}
	\label{FeatSelData}
	\setlength{\tabcolsep}{0.6em}
	\begin{sc}
		\begin{tabular}{ lccccccc }
			\toprule
			\multirow{2}{*}{Dataset} & \multirow{2}{*}{\#I} & \multirow{2}{*}{\#F} & \multirow{2}{*}{\#C} & \multicolumn{2}{c}{Selection} & \multicolumn{2}{c}{Ranking} \\
			\cmidrule(lr){5-6} \cmidrule(lr){7-8}
			&&&& RMI & ARMI & RMI & ARMI \\
			\midrule
			Madelon & 2600 & 500 & 2 & 9.98 & 1.33 & 1.10 & 0.23 \\
			Krvskp & 3196 & 37 & 2 & 1.37 & 0.14 & 0.15 & 0.03 \\
			Optdigits & 5620 & 65 & 10 & 13.79 & 0.80 & 1.44 & 0.14 \\
			Statlog & 6435 & 37 & 6 & 11.71 & 0.59 & 1.22 & 0.10 \\
			Spambase & 4601 & 57 & 2 & 6.64 & 0.48 & 0.70 & 0.08 \\
			Waveform & 5000 & 40 & 3 & 6.17 & 0.39 & 0.65 & 0.07 \\
			Galaxy & 9150 & 16 & 2 & 14.28 & 0.52 & 1.47 & 0.09 \\
			Beans & 13611 & 17 & 7 & 48.14 & 1.17 & 4.91 & 0.21 \\
			\bottomrule
		\end{tabular}
	\end{sc}
\end{table}
\begin{table*}[t]
	\centering
	\caption{The best classification error (\%) achieved by each feature ranking (the upper half) and feature selection (the lower half) methods for $k=10$ features. The last column indicates the average ranking of different methods on our test benchmark.}
	\label{FeatSel}
	\begin{sc}
		\begin{tabular}{ lccccccccc }
			\toprule
			& Madelon & Krvskp & Optdigits & Statlog & Spambase & Waveform & Galaxy & Beans & Average Rank \\
			\midrule
			ADC & 15.00 & 5.88 & 9.77 & 13.33 & 9.11 & 18.04 & 0.73 & 8.38 & 3.25 \\
			NFIG & 15.00 & 5.63 & 73.47 & 14.84 & 9.98 & 18.04 & 0.64 & 8.38 & 4.25 \\
			SU & 15.00 & 5.82 & 12.31 & 14.55 & 9.13 & 18.04 & 0.63 & 8.38 & 3.25 \\
			DAS & 48.81 & 38.96 & 88.67 & 15.70 & 24.82 & 18.04 & 0.81 & 8.38 & 5.88 \\
			WJE & 15.23 & 5.07 & 16.35 & 18.20 & 11.39 & 16.32 & 0.86 & 8.38 & 4.75 \\
			\midrule
			RMI & 13.23 & 5.48 & 7.31 & 14.83 & 9.78 & 16.14 & 0.51 & 8.29 & 1.63\\
			ARMI & 13.23 & 5.48 & 7.31 & 14.83 & 9.78 & 16.14 & 0.51 & 8.33 & 1.75 \\
			\midrule
			\midrule
			MIFS & 47.00 & 5.88 & 8.90 & 12.37 & 21.73 & 26.28 & 0.96 & 7.02 & 5.75 \\
			FOU & 34.42 & 4.79 & 10.28 & 11.83 & 18.04 & 19.70 & 12.77 & 7.05 & 5.25 \\
			MIM & 14.65 & 5.88 & 9.77 & 13.07 & 9.26 & 17.92 & 0.74 & 8.40 & 6.00 \\
			MRMR & 46.77 & 5.85 & 7.35 & 12.65 & 9.02 & 15.28 & 0.90 & 7.55 & 4.88 \\
			JMI & 12.27 & 5.88 & 6.55 & 12.77 & 9.17 & 14.86 & 0.69 & 8.96 & 4.63\\
			CMIM & 16.65 & 5.79 & 5.64 & 12.65 & 8.91 & 15.28 & 5.98 & 7.08 & 3.88 \\
			\midrule
			RMI & 10.62 & 2.57 & 5.50 & 12.34 & 8.87 & 15.62 & 0.51 & 7.11 & 2.00 \\
			ARMI & 10.31 & 2.57 & 5.91 & 12.63 & 8.87 & 16.82 & 0.51 & 7.41 & 2.63 \\
			\bottomrule
		\end{tabular}
	\end{sc}
\end{table*}

Feature ranking is a simple yet effective workaround. Using the weight $w_i = I_\alpha(X_i; Y)$ for the features $X_i \in S$, ranking methods are enabled to measure the effectiveness of each feature respectively and select the most important features upon two-dimensional probability distributions, which is much easier to estimate in practical. However, they ignore information redundancy and synergy between different features, thus usually yield suboptimal solutions.

We evaluate matrix-based R\'enyi's mutual information (RMI, (\ref{eq:mutual_info})) and our Approximated RMI (ARMI) with $5$ state-of-the-art information-based feature ranking methods, namely Asymmetric Dependency Coefficient (ADC) \cite{sridhar1998information}, Normalized First-Order Information Gain (NFIG) \cite{setiono1996improving}, Symmetrical Uncertainty (SU) \cite{press1988numerical}, Distance-based Attribute Selection (DAS) \cite{mantaras1991distance} and Weighted Joint Entropy (WJE) \cite{chi1993entropy}. 8 well-known classification datasets used in previous works constitute our experiment benchmark \cite{uci, vinh2016can, koklu2020multiclass, abolfathi2018fourteenth}, which covers a wide range of instance-feature ratios, number of classes, discreteness and data source domains as shown in Table \ref{FeatSelData}.

For non-R\'enyi methods, continuous features are discretized into $5$ bins by equal-width strategy used in \cite{vinh2014reconsidering}. We set the number of random vectors $s = 100$ in ARMI. The Support Vector Machine (SVM) algorithm with RBF kernel ($\sigma=1$) is adopted as the classifier using 10-fold cross-validation. In our observation, classification accuracy tends to stabilize after selecting the top $k=10$ features (shown in the Appendix), so we report the best accuracy achieved by each method for selecting at most $10$ features in Table \ref{FeatSel}. The comparison of running time between RMI and ARMI is reported in Table \ref{FeatSelData}. As we can see, both RMI and ARMI out-perform other Shannon's entropy based methods. Moreover, ARMI achieves $5$ to $20$ times speedup, $11.11$ times on average compared to the original RMI. For all datasets except ``Beans'', ARMI obtains exactly the same feature ranking orders as RMI.

\subsubsection{Application to Feature Selection}
We further explore the possibility of improving feature ranking methods by considering the interactions between different features, i.e. directly maximize our final target $I_\alpha(S_{sub}; Y)$. Before the proposition of RMI, this quantity was extremely hard or even intractable to estimate due to the curse of high dimensionality. Thus, enormous efforts have been made on approximation techniques that retains only the first or second order interactions, including Mutual Information-based Feature Selection (MIFS) \cite{battiti1994using}, First-Order Utility (FOU) \cite{brown2009new}, Mutual Information Maximization (MIM) \cite{lewis1992feature}, Maximum-Relevance Minimum-Redundancy (MRMR) \cite{peng2005feature}, Joint Mutual Information (JMI) \cite{yang1999data} and Conditional Mutual Information Maximization (CMIM) \cite{fleuret2004fast} that achieve the state-of-the-art performance.

We use the same datasets, discretization criterion and classifier settings. We follow a greedy strategy and select the first $10$ features that maximize the target $I_\alpha(S_{sub};Y)$. The results are shown in Table \ref{FeatSel} and \ref{FeatSelData}. ARMI achieves $10$ to $40$ times speedup, $19.07$ on average over original RMI. For ``Optdigits'', ``Spambase'' and ``Galaxy'' datasets, RMI and ARMI select exactly the same features. Also, it is worth noting that in the Galaxy dataset, there is one feature named "redshift" that achieves higher than $95\%$ classification accuracy solely in identifying the class of a given star. In both selection and ranking experiments, RMI and ARMI successfully identified this feature in the first place, but other methods failed to find it out until the third important feature is selected. This verifies the effectiveness of ARMI in both low and high dimensional circumstances, and demonstrates its great potential on variety of information-based tasks.

\section{Conclusion}
In this paper, we develop computationally efficient approximations for matrix-based R\'enyi's entropy, which achieve substantially lower complexity compared to the trivial eigenvalue approach. Through further adoption of Taylor and Chebyshev expansions, we support arbitrary values of $\alpha$. Statistical guarantees are established for all proposed algorithms and their optimality is proven by theoretical analysis. Large-scale simulation and real-world experiments are conducted to support our theoretical results. It is shown that our approximation algorithms bring tremendous speedup for a wide range of information-related tasks, while only introducing negligible loss in accuracy.


\appendix

\section*{Proof of Main Results}
For simplicity, we ignore some trivial cases in the following analysis, including 1) $\kappa = 1$, i.e. all eigenvalues of $A$ equal $1/n$; 2) $v = 1$, i.e. except for the largest eigenvalue $\lambda_0 = 1$, all other eigenvalues of $A$ equal $0$.
 
\subsection{Properties of $\tr(A^\alpha)$}
In the limit condition where the eigenspectrum of $A$ takes extreme value, the information potential $\tr(A^\alpha)$ could be expressed in terms of $u$ and $v$:
\begin{align}
	\tr(A^\alpha) &\in
	\begin{cases}
		[\mu, n^{1-\alpha}], & \mathrm{for}~\alpha < 1 \\
		[n^{1-\alpha}, \mu], & \mathrm{for}~\alpha > 1
	\end{cases}, \label{eq:ip_bound} \\
	where\enspace&\mu = \frac{1 - un}{v - u} \cdot v^\alpha + \frac{vn - 1}{v - u} \cdot u^\alpha. \nonumber
\end{align}
$\mu$ is the special case of $\tr(A^\alpha)$ when all eigenvalues of $A$ belongs to $\{u, v\}$. Some properties about $\mu$ are worthy to address for our following analysis.
\begin{proposition}
	\label{th:logmu_bound}
	Let $\mu$ be defined as in (\ref{eq:ip_bound}), then
	\begin{equation}
		\abs{\log\mu} = \Omega\prn*{\abs{\alpha-1}}, \quad\abs{\log\mu} = O\prn*{\abs{\alpha-1}\log n}. \label{eq:logmu_bound}
	\end{equation}
\end{proposition}
\begin{proof}
    If $u = 0$, the conclusion is obvious since we have $\mu = v^{\alpha-1}$ and $v \in (1/n, 1)$.
    
	Otherwise, let $\kappa = v/u$ be the condition number of $A$, then
	\begin{align*}
		\mu &= u^\alpha\frac{\kappa^\alpha-\kappa^\alpha un+\kappa un-1}{\kappa u-u} \\
		&=u^{\alpha-1}\prn*{\frac{\kappa(\kappa^{\alpha-1}-1)(1-un)}{\kappa-1}+1}.
	\end{align*}
	Ignoring the trivial case $\kappa=1$, for any constant number $\gamma \in (0, \infty)$ that satisfies $\kappa > \gamma + 1$, we have:
	\begin{equation*}
		\kappa^{\alpha-1}-1 \le \frac{\kappa(\kappa^{\alpha-1}-1)}{\kappa-1} \le \prn*{1+\frac{1}{\gamma}}(\kappa^{\alpha-1}-1).
	\end{equation*}
	Thus, we have shown that for all $\kappa \in (1, \infty)$:
	\begin{align*}
		\abs*{\log\mu} ={} &\Theta\big(\big|(\alpha-1)\log u \\
		&+ \log\prn*{(\kappa^{\alpha-1}-1)(1-un)+1}\big|\big).
	\end{align*}
	When $u \in (1/2n, 1/n)$, $1-un = O(1)$ and
	\begin{align*}
		\abs*{\log\mu} &= \Omega\prn*{\abs*{(\alpha-1)\log u}} \\
		&= \Omega\prn*{\abs*{(\alpha-1)\log n}},\\
		\abs*{\log\mu} &= O\prn*{\abs*{(\alpha-1)(\log u + \log k)}} \\
		&= O\prn*{\abs*{(\alpha-1)\log v}}.
	\end{align*}
	Otherwise when $u <= 1/2n$, $\abs*{\log\mu} = \Theta(\abs{(\alpha-1)\log v})$. 
	Combining with $v \in (1/n, 1)$, we finally get:
	\begin{equation*}
		\abs{\log\mu} = \Omega\prn*{\abs{\alpha-1}}, \quad\abs{\log\mu} = O\prn*{\abs{\alpha-1}\log n}.
	\end{equation*}
\end{proof}

\subsection{Proof of Proposition \ref{th:trace_upper_lower}}
\begin{proof}
	Let $Z$ be the output by applying algorithm $\mathcal{A}$ on $\tr(A^\alpha)$, then with probability at least $1 - \delta$, the following inequality holds:
	\begin{equation*}
		-\epsilon \cdot \tr(A^\alpha) \le Z - \tr(A^\alpha) \le \epsilon \cdot \tr(A^\alpha).
	\end{equation*}
	When $\alpha < 1$, $1 < \mu \le \tr(A^\alpha)$ by (\ref{eq:ip_bound}), i.e.
	\begin{align*}
		1-\epsilon &= \mu^{-\epsilon_0} \ge \tr^{-\epsilon_0}(A^\alpha), \\
		1+\epsilon < \frac{1}{1-\epsilon} &= \mu^{\epsilon_0} \le \tr^{\epsilon_0}(A^\alpha).
	\end{align*}
	Then the following inequality holds:
	\begin{align*}
		(\tr^{-\epsilon_0}(A^\alpha) - 1) \tr(A^\alpha) &\le Z - \tr(A^\alpha) \\
		&\le (\tr^{\epsilon_0}(A^\alpha) - 1) \tr(A^\alpha) \\
		\tr^{1-\epsilon_0}(A^\alpha) &\le Z \le \tr^{1+\epsilon_0}(A^\alpha) \\
		\tr^{-\epsilon_0}(A^\alpha) &\le \frac{Z}{\tr(A^\alpha)} \le \tr^{\epsilon_0}(A^\alpha).
	\end{align*}
	Now take log for both sides, we have:
	\begin{align*}
		\abs*{\log{\frac{Z}{\tr(A^\alpha)}}} &\le \epsilon_0\abs*{\log\tr(A^\alpha)} \\
		\frac{1}{1-\alpha} \abs*{\log{Z} - \log{{\tr(A^\alpha)}}} &\le \frac{\epsilon_0}{1-\alpha} \abs*{\log\tr(A^\alpha)} \\
		\abs*{\tilde{S}_\alpha(A) - S_\alpha(A)} &\le \epsilon_0 \cdot S_\alpha(A),
	\end{align*}
	where $\tilde{S}_\alpha(A) = \frac{1}{1-\alpha}\log Z$ is the estimate of $S_\alpha(A)$. Similarly, we can derive the same result for $\alpha > 1$.
	
	On the other hand, let $Z$ be the output by applying algorithm $\mathcal{A}$ on $S_\alpha(A)$, then with probability at least $1 - \delta$, we have
	\begin{equation*}
		\abs*{Z - S_\alpha(A)} \le \epsilon_0 \cdot S_\alpha(A).
	\end{equation*}
	When $\alpha < 1$, with the same steps as above we can get:
	\begin{align*}
		(\tr^{-\epsilon_0}(A^\alpha) - 1) \tr(A^\alpha) &\le \tilde{\tr}(A^\alpha) - \tr(A^\alpha) \\
		&\le (\tr^{\epsilon_0}(A^\alpha) - 1) \tr(A^\alpha),
	\end{align*}
	where $\tilde{\tr}(A^\alpha) = \exp((1-\alpha)Z)$ is the estimate of $\tr(A^\alpha)$. By (\ref{eq:ip_bound}) we have $n^{1-\alpha} \ge \tr(A^\alpha)$, then
	\begin{align*}
		\tr^{\epsilon_0}(A^\alpha) \le n^{\epsilon_0(1-\alpha)} &= 1+\epsilon, \\
		\tr^{-\epsilon_0}(A^\alpha) \ge n^{-\epsilon_0(1-\alpha)} &= \frac{1}{1+\epsilon} \ge 1-\epsilon.
	\end{align*}
	Combining the inequalities above gives:
	\begin{equation*}
		-\epsilon \cdot \tr(A^\alpha) \le \tilde{\tr}(A^\alpha) - \tr(A^\alpha) \le \epsilon \cdot \tr(A^\alpha).
	\end{equation*}
	We can get the same result for $\alpha > 1$, which finishes the proof.
\end{proof}

\subsection{Proof of Theorem \ref{th:int_upper}}
\begin{lemma} (Theorem 1 in \cite{meyer2021hutchpp})
	\label{th:hutchpp_upper}
	If Algorithm \ref{alg:hutchpp} is implemented with $s = O\prn*{\frac{1}{\epsilon}\sqrt{\log\prn*{\frac{1}{\delta}}} + \log\prn*{\frac{1}{\delta}}}$ matrix-vector multiplication queries, then for any positive semi-definite matrix $f(A)$, with probability at least $1 - \delta$, the output $Z$ satisfies:
	\begin{equation*}
		\abs*{Z - \tr\big(f(A)\big)} \le \epsilon \cdot \tr\big(f(A)\big).
	\end{equation*}
\end{lemma}

\begin{proof}
	Combining Proposition \ref{th:trace_upper_lower} and Lemma \ref{th:hutchpp_upper}, with probability at least $1 - \delta$ using $s = O\prn*{\frac{1}{\epsilon_0}\sqrt{\log\prn*{\frac{1}{\delta}}} + \log\prn*{\frac{1}{\delta}}}$ queries, Algorithm \ref{alg:int} returns an estimate $\tilde{S}_\alpha(A)$ that satisfies:
	\begin{equation*}
		\abs*{\tilde{S}_\alpha(A) - S_\alpha(A)} \le \epsilon \cdot S_\alpha(A),
	\end{equation*}
	where $\epsilon_0 = 1-\min(\mu,1/\mu)^{\epsilon}$. We can get the convergence rate of $s$ by further applying Proposition \ref{th:logmu_bound}:
	\begin{equation*}
		s = O\prn*{\frac{1}{\epsilon\abs{\alpha-1}}\sqrt{\log\prn*{\frac{1}{\delta}}} + \log\prn*{\frac{1}{\delta}}}.
	\end{equation*}
	For integer $\alpha \ge 2$, $\abs{\alpha-1} = \Theta(1)$.
\end{proof}

\subsection{Proof of Theorem \ref{th:taylor_upper}}

\begin{proof}
	Let $B = A/v-I_n$, $p_m(A) = v^\alpha \sum_{k=0}^m \binom{\alpha}{k}B^k$ and $Z$ be the estimate of $\tr(p_m(A))$ using Hutch++ algorithm.
	\begin{align}
		&\quad \abs{\tr(p_m(A)) - \tr(A^\alpha)} \nonumber \\
		&= v^\alpha \abs*{\sum_{k=m+1}^\infty \binom{\alpha}{k} \tr(B^k)} \nonumber \\
		&\le v^\alpha C\abs*{\sum_{k=0}^\infty \binom{\alpha}{k} \tr(B^{m+1} B^k)} \label{BinomIneq} \\
		&\le C\abs*{\prn*{\frac{u}{v}-1}^{m+1}} \tr(A^\alpha), \label{TraceIneq}
	\end{align}
	where $C = \max_{k \in [0, \infty]} \abs*{\binom{\alpha}{\ceil{\alpha}+k+1}/\binom{\alpha}{k}}$.
	
	(\ref{BinomIneq}) follows by noticing that:
	\begin{align*}
	    \left|\binom{\alpha}{m+k+1}\middle/\binom{\alpha}{k}\right| &= \abs*{\prod_{i=k}^{m+k} \frac{\alpha - i}{i + 1}} \le \abs*{\prod_{i=k}^{\ceil{\alpha}+k} \frac{\alpha - i}{i + 1}} \\
	    &= \left|\binom{\alpha}{\ceil{\alpha}+k+1}\middle/\binom{\alpha}{k}\right|, \\
	    \lim_{k \rightarrow \infty} \abs*{\prod_{i=k}^{\ceil{\alpha}+k} \frac{\alpha - i}{i + 1}} &= 1.
	\end{align*}
	Thus $C$ is a constant that depends only on $\alpha$.
	(\ref{TraceIneq}) follows by the von Neumann's trace inequality, that for any two positive semi-definite matrices $A$ and $B$, $\tr(AB) \le \sum_i \lambda_i(A)\lambda_i(B)$, where $\lambda_i(A)$ denotes the i-th singular value of $A$ (the same for $\lambda_i(B))$. By noticing that $\binom{\alpha}{k} \tr(B^k)$, $k \in [m, \infty]$ are either all positive semi-definite or all negative semi-definite by assuming $m > \alpha$, we have:
	\begin{align*}
		\abs{\tr(B^{m+1}B^k)} &\le \abs{u/v-1}^{m+1} \sum_i \lambda_i^k(B) \\
		&= \abs{u/v-1}^{m+1}\abs{\tr(B^k)}.
	\end{align*}
	Let $\epsilon_0 = 1-\min(\mu,1/\mu)^{\epsilon}$ and $\epsilon_1 = \frac{\epsilon_0}{3}$. by taking $m = O\prn*{\kappa\log\prn*{\frac{1}{\epsilon_0}}}$, we have
	\begin{align}
		C\abs*{\frac{u}{v} - 1}^{m+1} &< \frac{\epsilon_0}{2}, \label{eq:taylor_m} \\
		\abs{\tr(p_m(A)) - \tr(A^\alpha)} &\le \frac{\epsilon_0}{2} \tr(A^\alpha), \nonumber \\
		\tr\prn*{p_m(A)} \le \frac{\epsilon_0}{2} \tr(A^\alpha) &+ \tr(A^\alpha) \le \frac{3}{2} \tr(A^\alpha). \nonumber
	\end{align}
	Additionally, noticing that
	\begin{align}
	    \min_{\lambda \in [u, v]} p_m(\lambda) &\ge \min_{\lambda \in [u, v]} \prn*{\lambda^\alpha - \abs{p_m(\lambda) - \lambda^\alpha}} \nonumber\\
	    &\ge \min_{\lambda \in [u, v]} \prn*{\lambda^\alpha - C\abs*{\prn*{\frac{u}{v}-1}^{m+1}}\lambda^\alpha} \label{TaylorPositive1}\\
	    &\ge \min_{\lambda \in [u, v]} \prn*{\prn*{1 - \frac{\epsilon_0}{2}}\lambda^\alpha} > 0. \label{TaylorPositive2}
	\end{align}
	(\ref{TaylorPositive1}) follows by taking $A$ as a $1 \times 1$ matrix with entry $\lambda$ in (\ref{TraceIneq}).
	(\ref{TaylorPositive2}) follows by applying (\ref{eq:taylor_m}).
	Therefore, $p_m(A)$ is positive semi-definite when $m$ is large enough. By taking $s = O\prn*{\frac{1}{\epsilon_1}\sqrt{\log\prn*{\frac{1}{\delta}}} + \log\prn*{\frac{1}{\delta}}}$ in Lemma \ref{th:hutchpp_upper} we have:
	\begin{equation*}
		\abs*{Z - \tr\prn*{p_m(A)}} \le \frac{\epsilon_0}{3} \tr(p_m(A)).
	\end{equation*}
	Combining the results we get so far:
	\begin{align*}
		\abs*{Z - \tr(A^\alpha)} &\le
		\abs*{Z - \tr\prn*{p_m(A)}} \\
		&\quad+ \abs*{\tr\prn*{p_m(A)} - \tr(A^\alpha)} \\
		&\le \frac{\epsilon_0}{3} \tr\prn*{p_m(A)} + \frac{\epsilon_0}{2} \tr(A^\alpha) \\
		&\le \frac{\epsilon_0}{2} \tr(A^\alpha) + \frac{\epsilon_0}{2} \tr(A^\alpha) \\
		&= \epsilon_0 \cdot \tr(A^\alpha).
	\end{align*}
	Applying Proposition \ref{th:trace_upper_lower} and \ref{th:logmu_bound}, we finally have:
	\begin{align*}
		s &= O\prn*{\frac{1}{\epsilon\abs{\alpha-1}}\sqrt{\log\prn*{\frac{1}{\delta}}} + \log\prn*{\frac{1}{\delta}}}, \\
		m &= O\prn*{\kappa\log\prn*{\frac{1}{\epsilon\abs{\alpha-1}}}}.
	\end{align*}
\end{proof}

\subsection{Proof of Theorem \ref{th:taylor_upper_0}}
\begin{lemma} (Theorem 2 in \cite{das2019})
	\label{GammaIneq}
	Let $\Gamma(x)$ be the gamma function and let $R(x,y) = \Gamma(x+y)/\Gamma(x)$, then
	\begin{align*}
		R(x, y) &\ge x(x+y)^{y-1} &for\enspace 0\le y \le 1,\\
		R(x, y) &\ge x^y &for\enspace 1\le y \le 2,\\
		R(x, y) &\ge x(x+1)^{y-1} &for\enspace y \ge 2.
	\end{align*}
\end{lemma}

\begin{lemma} (Theorem 5 in \cite{meyer2021hutchpp})
	\label{th:hutchpp_upper_0}
	If Algorithm \ref{alg:hutchpp} is implemented with $s = O\prn*{\frac{1}{\epsilon}\sqrt{\log\prn*{\frac{1}{\delta}}} + \log\prn*{\frac{1}{\delta}}}$ matrix-vector multiplication queries, then for any matrix $f(A)$, with probability at least $1 - \delta$, the output $Z$ satisfies:
	\begin{equation*}
		\abs*{Z - \tr\big(f(A)\big)} \le \epsilon \cdot \norm{f(A)}_*,
	\end{equation*}
	where $\norm{\cdot}_*$ is the nuclear norm.
\end{lemma}

\begin{proof}
    Let $p_m(\lambda) = v^\alpha \sum_{k=0}^m \binom{\alpha}{k}(\lambda/v - 1)^k$ and $Z$ be the estimate of $\tr(p_m(A))$ using Hutch++ algorithm. Denote $E(\lambda)$ as the polynomial approximation error at point $\lambda$: $E(\lambda) = \abs{p_m(\lambda) - \lambda^\alpha}$. By noticing that $\binom{\alpha}{k} (\lambda/v-1)^k$, $k \in [m, \infty]$ are either all positive or all negative for $\lambda \in [0, v]$ by assuming $m > \alpha$, we have:
    \begin{align*}
        E(\lambda) &= v^\alpha \abs*{\sum_{k=m+1}^\infty \binom{\alpha}{k} \prn*{\frac{\lambda}{v}-1}^k} \\
        &\le v^\alpha \abs*{\sum_{k=m+1}^\infty \binom{\alpha}{k} (-1)^k} = E(0).
    \end{align*}
    From the property of binomial terms, we have that for any $\alpha > 0$ and integer $k > 1$, $\binom{\alpha}{k-1} + \binom{\alpha}{k} = \binom{\alpha+1}{k}$. Then
    \begin{align*}
        &\quad \prn*{1+\frac{\lambda}{v}-1}\sum_{k=m}^\infty \binom{\alpha}{k} \prn*{\frac{\lambda}{v}-1}^k \\
        &= \sum_{k=m}^\infty \binom{\alpha}{k} \prn*{\frac{\lambda}{v}-1}^k + \sum_{k=m+1}^\infty \binom{\alpha}{k-1} \prn*{\frac{\lambda}{v}-1}^k \\
        &= \binom{\alpha}{m} \prn*{\frac{\lambda}{v}-1}^m + \sum_{k=m+1}^\infty \binom{\alpha+1}{k} \prn*{\frac{\lambda}{v}-1}^k.
    \end{align*}
    Setting $\lambda = 0$ in the equation above, we have:
    \begin{equation*}
        \binom{\alpha}{m} (-1)^m + \sum_{k=m+1}^\infty \binom{\alpha+1}{k} (-1)^k = 0.
    \end{equation*}
    Therefore:
    \begin{align}
        E(0) &= v^\alpha \abs*{\sum_{k=m+1}^\infty \binom{\alpha}{k} (-1)^k} = v^\alpha \abs*{-\binom{\alpha-1}{m} (-1)^m}  \nonumber\\
        &= v^\alpha \abs*{\binom{\alpha-1}{m}} = v^\alpha \abs*{\frac{\Gamma(\alpha)}{\Gamma(m+1)\Gamma(\alpha-m)}} \nonumber\\
		&\le v^\alpha \abs*{\frac{\Gamma(\alpha)}{\Gamma(m-\alpha+1)\Gamma(\alpha-m)}} \nonumber\\
		&\qquad \cdot 
		\begin{cases}
			\frac{1}{(m-\alpha+1)(m+1)^{\alpha-1}} & 0 < \alpha < 1\\
			\frac{1}{(m-\alpha+1)^\alpha} & 1 < \alpha < 2\\
			\frac{1}{(m-\alpha+1)(m-\alpha+2)^{\alpha-1}} & \alpha \ge 2
		\end{cases}
		\label{GammaIneq1}\\
		&\le \frac{v^\alpha \Gamma(\alpha)}{\pi} \cdot \frac{2}{(m-\alpha+1)^\alpha} \label{GammaProp1}.
    \end{align}
	(\ref{GammaIneq1}) follows by applying Lemma \ref{GammaIneq} on $R(m-\alpha+1, \alpha)$. (\ref{GammaProp1}) follows by Euler's reflection formula that for any non-integer number $z$, $\Gamma(z)\Gamma(1-z)=\pi/\sin\pi z$. And by assuming that $m \ge 1$, $(m+1)^{1-\alpha} \le 2(m-\alpha+1)^{1-\alpha}$ when $\alpha < 1$.

    Let $\epsilon_0 = 1-\min(\mu,1/\mu)^{\epsilon}$ and $\epsilon_1 = \frac{\epsilon_0}{3}$. By choosing $m$ as:
    \begin{gather*}
        \frac{2v^\alpha\Gamma(\alpha)}{\pi(m-\alpha+1)^\alpha} \le \frac{\epsilon_0}{2n} \cdot \tr(A^\alpha), \\
        m \ge \alpha - 1 + \sqrt[\alpha]{\frac{4nv^\alpha\Gamma(\alpha)}{\epsilon_0\pi\min(v^{\alpha-1},n^{1-\alpha})}}.
    \end{gather*}
    Let $\lambda_1$, $\cdots$, $\lambda_n$ be the eigenvalues of $A$, then
    \begin{align*}
        \abs*{\tr\prn*{p_m(A)} - \tr(A^\alpha)} &\le \sum_{i=1}^n E(\lambda_i) \le nE(0) \\
        &\le \frac{2nv^\alpha \Gamma(\alpha)}{\pi(m-\alpha+1)^\alpha} \\
        &\le \frac{\epsilon_0}{2} \cdot \tr(A^\alpha).
    \end{align*}
    Taking $s = O\prn*{\frac{1}{\epsilon_1}\sqrt{\log\prn*{\frac{1}{\delta}}} + \log\prn*{\frac{1}{\delta}}}$ in Lemma \ref{th:hutchpp_upper_0} we have:
	\begin{align*}
		\abs*{Z - \tr\prn*{p_m(A)}} &\le \frac{\epsilon_0}{3} \norm{p_m(A)}_* = \frac{\epsilon_0}{3} \sum_{i=1}^n \abs{p_m(\lambda_i)} \\
		&\le \frac{\epsilon_0}{3} \prn*{\sum_{i=1}^n \lambda_i^\alpha + \sum_{i=1}^n \abs{p_m(\lambda_i) - \lambda_i^\alpha}} \\
		&= \frac{\epsilon_0}{3} \prn*{\tr(A^\alpha) + \sum_{i=1}^n E(\lambda_i)} \\
		&\le \frac{\epsilon_0}{3} \prn*{\tr(A^\alpha) + \frac{\epsilon_0}{2} \tr(A^\alpha)} \\
		&\le \frac{\epsilon_0}{2} \cdot \tr(A^\alpha).
	\end{align*}
	Combining the results we get so far:
	\begin{align*}
		\abs*{Z - \tr(A^\alpha)} &\le
		\abs*{Z - \tr\prn*{p_m(A)}} \\
		&\quad+ \abs*{\tr\prn*{p_m(A)} - \tr(A^\alpha)} \\
		&\le \frac{\epsilon_0}{2} \tr(A^\alpha) + \frac{\epsilon_0}{2} \tr(A^\alpha) \\
		&= \epsilon_0 \cdot \tr(A^\alpha).
	\end{align*}
	Applying Proposition \ref{th:trace_upper_lower} and \ref{th:logmu_bound}, we finally have:
	\begin{align*}
		s &= O\prn*{\frac{1}{\epsilon\abs{\alpha-1}}\sqrt{\log\prn*{\frac{1}{\delta}}} + \log\prn*{\frac{1}{\delta}}}, \\
		m &= \begin{cases}
			O\prn*{\sqrt[\alpha]{vn}\sqrt[\alpha]{\frac{1}{\epsilon\abs{\alpha-1}}}} & \alpha < 1 \\
			O\prn*{vn\sqrt[\alpha]{\frac{1}{\epsilon\abs{\alpha-1}}}} & \alpha > 1
		\end{cases}.
	\end{align*}
\end{proof}

\subsection{Proof of Theorem \ref{th:cheby_upper}}
The following lemma gives the upper bound of Chebyshev series approximation.
\begin{lemma} (Theorem 2.1 in \cite{xiang2010error})
	\label{th:cheby_upper_lemma}
	Suppose $f$ is analytic with $\abs{f(z)} \le M$ in the region bounded by the ellipse with foci $\pm1$ and major and minor semi-axis lengths summing to $K > 1$. Let $p_m$ denote the Chebyshev polynomial approximation of $f$ with degree $m$, then for any $m \in \mathbb{Z}^+$:
	\begin{equation*}
		\max_{\lambda \in [-1, 1]} \abs*{f(\lambda) - p_m(\lambda)} \le \frac{4M}{(K-1)K^m}.
	\end{equation*}
\end{lemma}

By selecting an appropriate analytic region, we are able to establish the error bound of approximating $f(\lambda) = \lambda^\alpha$.
\begin{proposition}
	\label{th:cheby_upper_prop}
	Let $g$ be the linear mapping $[-1, 1] \rightarrow [u, v]$, $f(\lambda) = \lambda^\alpha$ be the target function, $p_m(\lambda)$ be the Chebyshev series of degree $m = O\prn*{\sqrt{\frac{v}{u}}\log\prn*{\frac{v}{u\epsilon}}}$ for function $f \circ g$, then the following inequality holds:
	\begin{align*}
		&\max_{x \in [-1, 1]} \abs*{(f \circ g)(x) - p_m(x)} \\
		&\quad= \max_{\lambda \in [u, v]} \abs*{f(\lambda) - q_m(\lambda)} \le \epsilon u^\alpha,
	\end{align*}
	where $q_m = p_m \circ g^{-1}$.
\end{proposition}
\begin{proof}
	The analyticity of $f$ can be acquired by analyzing the binomial series:
	\begin{equation*}
		(1+z)^\alpha = \sum_{k=0}^\infty \binom{\alpha}{k} z^k.
	\end{equation*}
	When $\alpha \in \mathbb{R}^+$ and $\alpha \notin \mathbb{N}$, binomial series converges absolutely if and only if $\abs{z} \le 1$, which indicates that function $f(z) = z^\alpha$ is analytic in the circle located at $1 + 0i$ with radius $1$. Combining with linear mapping $g$, the function $f \circ g$ is analytic in the ellipse located at $\frac{2-v-u}{v-u} + 0i$ with major semi-axis length $\frac{2}{v-u}$ and minor semi-axis length $1$, denoted by $E_g$.
	
	Next we will choose an ellipse region $E_c$ within $E_g$ whose foci are at $\pm 1$. By noticing that $E_g$ intersects with real-axis at $\frac{2u}{v-u} + 0i$, we choose major semi-axis length $\beta = \frac{2u}{v - u}$ for $E_c$. By selecting $v \ge u + 2\sqrt{2u-u^2}$, region $E_c$ is inside $E_g$.
	
	We then apply Lemma \ref{th:cheby_upper_lemma} with $K = 1 + \beta + \sqrt{\beta^2 + 2\beta}$ and $M = \prn*{1 + \beta}^\alpha$. By noticing that $\log K = \Theta\prn*{\sqrt{\beta}}$, we get the upper bound of $m$:
	\begin{equation*}
		m \ge \frac{\log\prn*{\frac{4M}{(K-1)\epsilon u^\alpha}}}{\log K} = O\prn*{\sqrt{\frac{v}{u}}\log\prn*{\frac{v}{u\epsilon}}}.
	\end{equation*}
\end{proof}

\begin{proof}
	By taking $m = O\prn*{\sqrt{\frac{v}{u}}\log\prn*{\frac{v}{u\epsilon_1}}}$ in Proposition \ref{th:cheby_upper_prop}, where $\epsilon_1 = \epsilon_0 / 2$ and $\epsilon_0 = 1-\min(\mu,1/\mu)^{\epsilon}$, we have:
	\begin{align*}
		\max_{\lambda \in [u, v]} \abs*{f(\lambda) - q_m(\lambda)} &\le \epsilon_1 u^\alpha, \\
		\abs*{\tr\prn*{q_m(A)} - \tr(A^\alpha)} &\le \sum_{i=1}^n \abs*{f(\lambda_i) - q_m(\lambda_i)} \\
		&\le n\epsilon_1 u^\alpha \le \frac{\epsilon_0}{2} \cdot \tr(A^\alpha).
	\end{align*}
	Additionally, noticing that
	\begin{align*}
	    \min_{\lambda \in [u, v]} q_m(\lambda) &\ge \min_{\lambda \in [u, v]} \lambda^\alpha - \max_{\lambda \in [u, v]} \abs{\lambda^\alpha - q_m(\lambda)} \\
	    &\ge u^\alpha - \frac{\epsilon_0}{2} u^\alpha \ge 0.
	\end{align*}
	Therefore, $q_m(A)$ is positive semi-definite when $m$ is large enough. By taking $s = O\prn*{\frac{1}{\epsilon_2}\sqrt{\log\prn*{\frac{1}{\delta}}} + \log\prn*{\frac{1}{\delta}}}$ in Lemma \ref{th:hutchpp_upper} where $\epsilon_2 = \frac{\epsilon_0}{3}$ we have:
	\begin{align*}
		\abs*{Z - \tr\prn*{q_m(A)}} &\le \frac{\epsilon_0}{3} \tr(q_m(A)), \\
		\tr\prn*{q_m(A)} \le \frac{\epsilon_0}{2} \tr(A^\alpha) &+ \tr(A^\alpha) \le \frac{3}{2} \tr(A^\alpha),
	\end{align*}
	where $Z$ is the estimate of $\tr\prn*{q_m(A)}$ using Hutch++ algorithm. Combining the results we get so far:
	\begin{align*}
		\abs*{Z - \tr(A^\alpha)} &\le
		\abs*{Z - \tr\prn*{q_m(A)}} \\
		&\quad+ \abs*{\tr\prn*{q_m(A)} - \tr(A^\alpha)} \\
		&\le \frac{\epsilon_0}{3} \tr\prn*{q_m(A)} + \frac{\epsilon_0}{2} \tr(A^\alpha) \\
		&\le \frac{\epsilon_0}{2} \tr(A^\alpha) + \frac{\epsilon_0}{2} \tr(A^\alpha) \\
		&\le \epsilon_0 \cdot \tr(A^\alpha).
	\end{align*}
	Applying Proposition \ref{th:trace_upper_lower} and \ref{th:logmu_bound}, we finally have:
	\begin{align*}
		s &= O\prn*{\frac{1}{\epsilon\abs{\alpha-1}}\sqrt{\log\prn*{\frac{1}{\delta}}} + \log\prn*{\frac{1}{\delta}}}, \\
		m &= O\prn*{\sqrt{\kappa}\log\prn*{\frac{\kappa}{\epsilon\abs{\alpha-1}}}}.
	\end{align*}
\end{proof}

\subsection{Proof of Theorem \ref{th:cheby_upper_0}}
\begin{proof}
	When $u = 0$, the coefficients of Chebyshev series $\hat{T}_k$ have analytical expressions:
	\begin{align*}
		c_k &= \frac{2}{\pi} \int_0^\pi (q_m)^\alpha(cos\theta) \cos(k\theta) \dif\theta \\
		&= \frac{2}{\pi} \int_0^\pi \prn*{\frac{v}{2}(\cos\theta+1)}^\alpha \cos(k\theta) \dif\theta \\
		&= \frac{2v^\alpha \Gamma(\alpha+\frac{1}{2}) (\alpha)_k}{\sqrt{\pi} \Gamma(\alpha+1) (\alpha+k)_k}.
	\end{align*}
	where $(\alpha)_k$ is the falling factorial: $(\alpha)_k = \alpha \cdot(\alpha-1)\cdot ... \cdot(\alpha-k+1)$. Then for each eigenvalue $\lambda$ of $A$:
	\begin{align}
		&\quad\abs{\lambda^\alpha - q_m(\lambda)} = \abs*{\sum_{i=m+1}^\infty c_i \hat{T}_i(\lambda)} \nonumber\\
		&\le \sum_{i=m+1}^\infty \abs{c_i} = \sum_{i=m+1}^\infty \abs*{\frac{2v^\alpha \Gamma(\alpha + \frac{1}{2}) (\alpha)_i}{\sqrt{\pi} \Gamma(\alpha + 1) (\alpha+i)_i}} \label{ChebyProp1}\\
		&= \frac{2v^\alpha}{\sqrt{\pi}}\sum_{i=m+1}^\infty \abs*{\frac{\Gamma(\alpha + \frac{1}{2}) \Gamma(\alpha + 1)}{\Gamma(\alpha + i + 1) \Gamma(\alpha - i + 1)}} \nonumber\\
		&\le \frac{2v^\alpha}{\sqrt{\pi}}\sum_{i=m+1}^\infty \abs*{\frac{\Gamma(\alpha + \frac{1}{2}) \Gamma(\alpha + 1)}{\Gamma(i-\alpha)\Gamma(\alpha-i+1) (i-\alpha)^{2\alpha+1}}} \label{GammaIneq2}\\
		&\le \frac{2v^\alpha \Gamma(\alpha + \frac{1}{2}) \Gamma(\alpha + 1)}{\pi^{3/2}}\sum_{i=m+1}^\infty \abs*{\frac{1}{(i-\alpha)^{2\alpha+1}}} \label{GammaProp2}\\
		&\le \frac{2v^\alpha \Gamma(\alpha + \frac{1}{2}) \Gamma(\alpha + 1)}{\pi^{3/2}} \int_m^\infty \frac{1}{(x-\alpha)^{2\alpha+1}} \dif x \label{IntPoly}\\
		&= \frac{2v^\alpha \Gamma(\alpha + \frac{1}{2}) \Gamma(\alpha + 1)}{\pi^{3/2}} \frac{1}{2\alpha(m-\alpha)^{2\alpha}} \nonumber\\
		&= \frac{v^\alpha \Gamma(\alpha + \frac{1}{2}) \Gamma(\alpha)}{\pi^{3/2} (m-\alpha)^{2\alpha}}. \nonumber
	\end{align}
	(\ref{ChebyProp1}) follows by noticing that $\hat{T}_n(x) \in [-1,1]$ for any $x \in [0,v]$. (\ref{GammaIneq2}) follows by applying Lemma \ref{GammaIneq} on $R(i-\alpha,2\alpha+1)$ similar to (\ref{GammaIneq1}). (\ref{GammaProp2}) follows by Euler's reflection formula similar to (\ref{GammaProp1}). (\ref{IntPoly}) follows by assuming $m > \alpha$ and noticing that $n^{-k} \le \int_{n-1}^n x^{-k} \dif x$ for $n > 1$ and $k > 1$.
	
	Let $\epsilon_0 = 1-\min(\mu,1/\mu)^{\epsilon}$ and $\epsilon_1 = \frac{\epsilon_0}{3}$. By choosing $m$ as:
	\begin{gather*}
		\frac{nv^\alpha \Gamma(\alpha + \frac{1}{2}) \Gamma(\alpha)}{\pi^{3/2} (m-\alpha)^{2\alpha}} \le \frac{\epsilon_0}{2} \cdot \tr(A^\alpha), \\
		m \ge \alpha + \sqrt[2\alpha]{\frac{2 nv^\alpha \Gamma(\alpha + \frac{1}{2}) \Gamma(\alpha)}{\epsilon_0 \pi^{3/2} \min(v^{\alpha-1}, n^{1-\alpha})}}.
	\end{gather*}
	Let $\lambda_1$, $\cdots$, $\lambda_n$ be the eigenvalues of $A$, then
	\begin{equation*}
		\abs*{\tr\prn*{q_m(A)} - \tr(A^\alpha)} \le \frac{\epsilon_0}{2} \cdot \tr(A^\alpha),
	\end{equation*}
	Taking $s = O\prn*{\frac{1}{\epsilon_1}\sqrt{\log\prn*{\frac{1}{\delta}}} + \log\prn*{\frac{1}{\delta}}}$ in Lemma \ref{th:hutchpp_upper_0} we have:
	\begin{align*}
		\abs*{Z - \tr\prn*{q_m(A)}} &\le \frac{\epsilon_0}{3} \norm{q_m(A)}_* = \frac{\epsilon_0}{3} \sum_{i=1}^n \abs{q_m(\lambda_i)} \\
		&\le \frac{\epsilon_0}{3} \prn*{\sum_{i=1}^n \lambda_i^\alpha + \sum_{i=1}^n \abs{\lambda_i^\alpha - q_m(\lambda_i)}} \\
		&\le \frac{\epsilon_0}{3} \prn*{\tr(A^\alpha) + \frac{\epsilon_0}{2} \tr(A^\alpha)} \\
		&\le \frac{\epsilon_0}{2} \cdot \tr(A^\alpha).
	\end{align*}
	Combining the results we get so far:
	\begin{align*}
		\abs*{Z - \tr(A^\alpha)} &\le
		\abs*{Z - \tr\prn*{p_m(A)}} \\
		&\quad+ \abs*{\tr\prn*{p_m(A)} - \tr(A^\alpha)} \\
		&\le \frac{\epsilon_0}{2} \tr(A^\alpha) + \frac{\epsilon_0}{2} \tr(A^\alpha) \\
		&= \epsilon_0 \cdot \tr(A^\alpha).
	\end{align*}
	Applying Proposition \ref{th:trace_upper_lower} and \ref{th:logmu_bound}, we finally have:
	\begin{align*}
		s &= O\prn*{\frac{1}{\epsilon\abs{\alpha-1}}\sqrt{\log\prn*{\frac{1}{\delta}}} + \log\prn*{\frac{1}{\delta}}}, \\
		m &= \begin{cases}
			O\prn*{\sqrt[2\alpha]{vn}\sqrt[2\alpha]{\frac{1}{\epsilon\abs{\alpha-1}}}} & \alpha < 1 \\
			O\prn*{\sqrt{vn}\sqrt[2\alpha]{\frac{1}{\epsilon\abs{\alpha-1}}}} & \alpha > 1
		\end{cases}.
	\end{align*}
\end{proof}

\subsection{Proof of Theorem \ref{th:entropy_lower}}
\begin{lemma} {Theorem 7 in \cite{meyer2021hutchpp}}
	\label{th:hutchpp_lower}
	Any algorithm that accesses a positive semi-definite matrix $A$ via matrix-vector multiplication queries $Ar_1, \cdots, Ar_m$, where $r_1, \cdots, r_m$ are possibly adaptively chosen vectors with integer entries in $\brk{-2^b, \cdots, 2^b}$, requires $s = \Omega\prn*{\frac{1}{\epsilon\prn*{b+\log\prn*{1/\epsilon}}}}$ such queries to output an estimate $Z$ so that, with probability $> \frac{2}{3}$, $(1-\epsilon)\,\tr(A) \le Z \le (1+\epsilon)\,\tr(A)$.
\end{lemma}

\begin{proof}
	Combining Lemma \ref{th:hutchpp_lower} and Proposition \ref{th:trace_upper_lower}, we have that the lower bound of estimating $S_\alpha(A)$ to relative error $1 \pm \epsilon$ with probability at least $\frac{2}{3}$ is
	\begin{equation*}
		s = \Omega\prn*{\frac{1}{\epsilon_0\prn*{b+\log\prn*{1/\epsilon_0}}}},
	\end{equation*}
	where $\epsilon_0 = n^{\epsilon\abs{\alpha-1}}-1 \le \epsilon\abs{1-\alpha}\log n$.
	
	In limited precision computation settings, $b$ is some constant value, then we finally get
	\begin{equation*}
		s = \Omega\prn*{\frac{\textstyle 1}{\epsilon\abs{\alpha-1}\log n\log\prn*{\frac{1}{\epsilon\abs{\alpha-1}\log n}}}}.
	\end{equation*}
\end{proof}

\subsection{Proof of Theorem \ref{th:poly_lower}}
The following lemma gives the convergence rate of function $f(x) = (\gamma-x)^{-t}$ in best uniform approximation. It is proved in \cite{lam1972some} pp.38-39 and \cite{bernsteincollected} pp.102-103.
\begin{lemma}
	\label{th:poly_lower_lemma}
	Let $\norm{\cdot}$ denote the $L_\infty$ norm of functions and $E_m(f) = \min_{p \in \mathbb{P}_m} \norm*{f - p}$ be the error of best uniform approximation of a given function $f(x)$ on the finite interval $[-1, 1]$. Then when $m \rightarrow \infty$,
	\begin{equation*}
		E_m\prn*{(\gamma - x)^{-t}} \sim \frac{m^{t-1}}{\abs*{\Gamma(t)}} \frac{\prn*{\gamma - \sqrt{\gamma^2-1}}^m}{\prn*{\sqrt{\gamma^2-1}}^{1+t}},
	\end{equation*}
	where $t, \gamma \in \mathbb{R}$ and $\gamma > 1$.
\end{lemma}

\begin{proposition}
	\label{th:poly_lower_prop}
	There exists a positive decreasing function $\epsilon_0: \mathbb{R}^+ \rightarrow \mathbb{R}^+$ such that for arbitrary $0 < u < v < 1$ and $\epsilon \in (0, \epsilon_0(v/u))$, any polynomial $q_m(\lambda)$ that approximates function $f(\lambda)=\lambda^\alpha$, requires $m = \Omega\prn*{\sqrt{\frac{v}{u}} \log\prn*{\frac{u}{v\epsilon}}}$ degree to achieve $\abs*{\sum_{i=1}^n (f(\lambda_i) - q_m(\lambda_i))} \le \epsilon$, for arbitrary real numbers $\lambda_1, \cdots, \lambda_n \in [u, v]$ that satisfy $\sum_{i=1}^n \lambda_i \in [b, b+v)$, where $b \ge v$ is a constant number.
\end{proposition}
\begin{proof}
	Under the same assumptions, We list the following problems for polynomial approximation. We claim that each of them could be reduced into the next problem in sequence.
	\begin{problem} \label{prb:lower1}
	    The minimum degree $m$ required for any polynomial $q_m$ to achieve $\abs*{\sum_{i=1}^n (f(\lambda_i) - q_m(\lambda_i))} \le \epsilon$ for any $\lambda_1, \cdots, \lambda_n \in [u, v]$ that satisfy $\sum_{i=1}^n \lambda_i \in [b, b+v)$.
	\end{problem}
	\begin{problem} \label{prb:lower2}
	    The minimum degree $m$ required for any polynomial $q_m$ to achieve $\sum_{i=1}^n \abs*{f(\lambda_i) - q_m(\lambda_i)} \le \epsilon$ for any $\lambda_1, \cdots, \lambda_n \in [u, v]$ that satisfy $\sum_{i=1}^n \lambda_i = b$.
	\end{problem}
	\begin{problem} \label{prb:lower3}
	    The minimum degree $m$ required for any polynomial $q_m$ to achieve $\abs*{f(\lambda) - q_m(\lambda)} \phi(\lambda) \le \epsilon$ for any $\lambda \in [u, v]$, where $\phi(\lambda) = \flr{\min(\frac{nv-b}{\lambda-u}, \frac{b-nu}{v-\lambda})}$.
	\end{problem}
	\begin{problem} \label{prb:lower4}
	    The minimum degree $m$ required for any polynomial $q_m$ to achieve $\abs*{f(\lambda) - q_m(\lambda)} \le \epsilon$ for any $\lambda \in [u, v]$.
	\end{problem}
	\begin{problem} \label{prb:lower5}
	    The minimum degree $m$ required for any polynomial $q_m$ to achieve $\abs*{f(\lambda) - q_m(\lambda)} \le \epsilon$ for any $\lambda \in [u, u + 2]$.
	\end{problem}
	For Problem \ref{prb:lower5}, approximating $(f \circ g)(x)=(x + u + 1)^\alpha$ is equivalent to approximating $(\gamma - x)^{-t}$ with $\gamma = u+1$ and $t = -\alpha$, since they are symmetric about y-axis. Let $\epsilon = E_m(f \circ g)$, with the following property of the gamma function
	\begin{equation*}
		\abs{\Gamma(-\alpha)} = \frac{\pi}{\abs{\sin\pi\alpha}\Gamma(\alpha+1)},
	\end{equation*}
	by applying Lemma \ref{th:poly_lower_lemma}, when $m$ is large enough, we have:
	\begin{align*}
		\lim_{m \rightarrow \infty} \frac{\prn*{1+u-\sqrt{u^2+2u}}^m}{m^{\alpha+1}} &= \frac{\epsilon\abs{\Gamma(-\alpha)}}{\prn*{\sqrt{u^2+2u}}^{\alpha-1}} \\
		&= \Theta\prn*{\frac{\epsilon}{(\sqrt{u})^{\alpha-1}\abs{sin\pi\alpha}}}.
	\end{align*}
	Thus, for each $u \in (0, 1)$, there is an $\epsilon_0 \in (0, 1)$, such that when $\epsilon < \epsilon_0$:
	\begin{equation*}
		m = \Omega\prn*{\frac{1}{\sqrt{u}} \log\prn*{\frac{u\abs{\sin\pi\alpha}}{\epsilon}}}.
	\end{equation*}
	When $\alpha \notin \mathbb{N}$, we have $\abs*{\sin\pi\alpha} = \Theta(1)$.
	
	Problem \ref{prb:lower4} $\rightarrow$ \ref{prb:lower5}: Noticing that an approximation of $f_{[u,u+2]}(\lambda)$ could be acquired by approximating $\prn*{\frac{u+2}{v}}^\alpha f_{[\frac{uv}{u+2},v]}\prn*{\frac{v}{u+2}\lambda}$, The lower bound of $m$ for approximating $f_{[\frac{uv}{u+2},v]}$ is then
	\begin{equation*}
		m = \Omega\prn*{\frac{1}{\sqrt{u}} \log\prn*{\frac{u}{v\epsilon}}}.
	\end{equation*}
	This is equivalent to approximating $f_{[u,v]}$ with
	\begin{equation*}
		m = \Omega\prn*{\sqrt{\frac{v}{u}} \log\prn*{\frac{u}{v\epsilon}}}.
	\end{equation*}
	
	Problem \ref{prb:lower3} $\rightarrow$ \ref{prb:lower4}: We can assume $u < \frac{b-v}{n-1}$ and $v > \frac{b-u}{n-1}$ without loss of generality, since otherwise it means $\min_i \lambda_i > u$ or $\max_i \lambda_i < v$, and the interval $[u, v]$ could be make tighter to fit this assumption. Then for any $\lambda \in [u, v]$,
	\begin{align*}
		\phi(\lambda) &= \flr*{\min\prn*{\frac{nv-b}{v-\lambda}, \frac{b-nu}{\lambda-u}}} \\
		&\ge \flr*{\min\prn*{\frac{nv-b}{v-u}, \frac{b-nu}{v-u}}} \ge 1.
	\end{align*}
	
	Problem \ref{prb:lower2} $\rightarrow$ \ref{prb:lower3}: For any $\lambda \in [u, v]$, we construct a series of numbers $\lambda_1$, $\cdots$, $\lambda_n$ by setting $\lambda_1, \cdots, \lambda_{n_\lambda} = \lambda$ and $\lambda_{n_\lambda+1}, \cdots, \lambda_n = \frac{b-\lambda n_\lambda}{n - n_\lambda}$, where $n_\lambda = \flr{\min(\frac{nv-b}{\lambda-u}, \frac{b-nu}{v-\lambda})}$. Then the series satisfies $\sum_{i=1}^n \lambda_i = b$ and
	\begin{align*}
		\sum_{i=1}^n \abs*{f(\lambda_i) - q_m(\lambda_i)} &\le \epsilon, \\
		n_\lambda \abs*{f(\lambda) - q_m(\lambda)} &\le \epsilon, \\
		\phi(\lambda) \abs*{f(\lambda) - q_m(\lambda)} &\le \epsilon.
	\end{align*}
	
	Problem \ref{prb:lower1} $\rightarrow$ \ref{prb:lower2}: Let $q_m$ be the solution for Problem \ref{prb:lower1} that achieves $\epsilon/2$ approximation error. In the trivial case, we have $f(\lambda) \le q_m(\lambda)$ for all $\lambda \in [u, v]$ (or $f(\lambda) \ge q_m(\lambda)$), then $\sum_{i=1}^n \abs*{f(\lambda_i) - q_m(\lambda_i)} = \abs*{\sum_{i=1}^n (f(\lambda_i) - q_m(\lambda_i))}$.
	
	Otherwise there exists some $\rho \in (u, v)$ that satisfies $f(\rho) = q_m(\rho)$, since both $f(\lambda)$ and $q_m(\lambda)$ are continuous functions. Given arbitrary query $\lambda_1$, $\cdots$, $\lambda_n$, there exists a partition $n_\rho$ of the reordered numbers $\lambda_i$ so that for all $i \in [1, n_\rho]$, $f(\lambda_i) \le q_m(\lambda_i)$ and for all $i \in [n_\rho + 1, n]$, $f(\lambda_i) > q_m(\lambda_i)$.
	\begin{align*}
		\sum_{i=1}^{n_\rho} \abs*{f(\lambda_i) - q_m(\lambda_i)} &= \abs*{\sum_{i=1}^{n_\rho} (f(\lambda_i) - q_m(\lambda_i))}, \\
		\sum_{i=n_\rho+1}^n \abs*{f(\lambda_i) - q_m(\lambda_i)} &= \abs*{\sum_{i=n_\rho+1}^n (f(\lambda_i) - q_m(\lambda_i))}.
	\end{align*}
	Construct two queries: $\lambda_1^1, \cdots, \lambda_{n_1}^1$ and $\lambda_1^2, \cdots, \lambda_{n_2}^2$:
	\begin{align*}
	    \lambda_i^1 = \begin{cases}
	        \lambda_i & i \le n_\rho \\
	        \rho & i > n_\rho
	    \end{cases}, &\quad n_1 = n_\rho + \ceil*{\sum_{i=n_\rho+1}^n \lambda_i / \rho}. \\
	    \lambda_i^2 = \begin{cases}
	        \lambda_{i+n_\rho} & i \le n - n_\rho \\
	        \rho & i > n - n_\rho
	    \end{cases}, &\quad n_2 = n - n_\rho + \ceil*{\sum_{i=1}^{n_\rho} \lambda_i / \rho}.
	\end{align*}
	Let $b = \sum_{i=1}^n \lambda_i$, then $\sum_{i=1}^{n_1} \lambda_i^1 \in [b, b+v)$ and $\sum_{i=1}^{n_2} \lambda_i^2 \in [b, b+v)$. Therefore
	\begin{align*}
	    &\quad \sum_{i=1}^n \abs*{f(\lambda_i) - q_m(\lambda_i)} \\
	    &= \sum_{i=1}^{n_\rho} \abs*{f(\lambda_i) - q_m(\lambda_i)} + \sum_{i={n_\rho+1}}^{n} \abs*{f(\lambda_i) - q_m(\lambda_i)} \\
	    &= \abs*{\sum_{i=1}^{n_\rho} (f(\lambda_i) - q_m(\lambda_i))} + \abs*{\sum_{i=n_\rho+1}^n (f(\lambda_i) - q_m(\lambda_i))} \\
	    &= \abs*{\sum_{i=1}^{n_1} (f(\lambda_i^1) - q_m(\lambda_i^1))} + \abs*{\sum_{i=1}^{n_2} (f(\lambda_i^2) - q_m(\lambda_i^2))} \\
	    &\le \frac{\epsilon}{2} + \frac{\epsilon}{2} = \epsilon.
	\end{align*}
	
	From the reductions above, the lower bound of $m$ for solving Problem \ref{prb:lower1} is $m = \Omega\prn*{\sqrt{\frac{v}{u}} \log\prn*{\frac{u}{v\epsilon}}}$.
\end{proof}

\begin{proof}
	Let $Z = \tr\prn*{q_m(A)}$ be the trace of the approximated matrix functional, then
	\begin{equation*}
		\abs*{Z - \tr(A^\alpha)} = \abs*{\sum_{i=1}^n (q_m(\lambda_i) - \lambda_i^\alpha)}.
	\end{equation*}
	Let $\epsilon_0 = n^{\epsilon\abs{\alpha-1}}-1$, then by applying Proposition \ref{th:poly_lower_prop}, we have that $m = \Omega\prn*{\sqrt{\frac{v}{u}} \log\prn*{\frac{u}{v\epsilon_0}}}$ is the lower bound to achieve $\abs*{Z - \tr(A^\alpha)} \le \epsilon_0$ with $b = 1$. Combining with Proposition \ref{th:trace_upper_lower}, the lower bound for matrix function approximation is
	\begin{align*}
		m &= \Omega\prn*{\sqrt{\frac{v}{u}} \log\prn*{\frac{u}{v\epsilon_0}}} \\
		&= \Omega\prn*{\sqrt{\frac{v}{u}} \log\prn*{\frac{u}{v\epsilon\abs{\alpha-1}\log n}}}.
	\end{align*}
\end{proof}

\subsection{Proof of Theorem \ref{th:poly_lower_0}}
Similarly, the following lemma gives the convergence rate of function $f(x) = \abs{x}^\alpha$ in best uniform approximation. It is proved in \cite{bernstein1938meilleure} and \cite{varga1992some}.
\begin{lemma}
	\label{th:poly_lower_lemma_0}
	Let $\norm{\cdot}$ denote the $L_\infty$ norm of functions and $E_m(f) = \min_{p \in \mathbb{P}_m} \norm*{f - p}$ be the error of best uniform approximation of a given function $f(x)$ on the finite interval $[-1, 1]$. Then when $m \rightarrow \infty$,
	\begin{equation*}
		E_m\prn*{x^\alpha} \sim \delta(\alpha) m^{-2\alpha},
	\end{equation*}
	where $\alpha \in \mathbb{R}^+$, $\delta(\alpha)$ is a non-negative constant number depending only on $\alpha$, and satisfies $\delta(\alpha) > 0$ when $\alpha \notin \mathbb{N}$.
\end{lemma}

\begin{proposition}
	\label{th:poly_lower_prop_0}
	For arbitrary $v > 0$ and small enough $\epsilon$, any polynomial $q_m(\lambda)$ that approximates $f(\lambda)=\lambda^\alpha$ requires $m = \Omega\prn*{\sqrt[2\alpha]{\frac{1}{\epsilon}}}$ degree to achieve $\abs*{\sum_{i=1}^n (f(\lambda_i) - q_m(\lambda_i))} \le \epsilon$ for arbitrary real numbers $\lambda_1, \cdots, \lambda_n \in [0, v]$ that satisfy $\sum_{i=1}^n \lambda_i \in [b, b+v)$, where $b \ge v$ is a constant number.
\end{proposition}
\begin{proof}
	We can prove Proposition 
	\ref{th:poly_lower_prop_0} through a similar procedure as the proof of \ref{th:poly_lower_prop}.
\end{proof}

\begin{proof}
    Let $Z = \tr\prn*{q_m(A)}$ be the trace of the approximated matrix functional and $\epsilon_0 = n^{\epsilon\abs{\alpha-1}}-1$. From Proposition \ref{th:poly_lower_prop_0} we know that $m = \Omega\prn*{\sqrt[2\alpha]{\frac{1}{\epsilon_0}}}$ is the lower bound to achieve $\abs*{Z - \tr(A^\alpha)} \le \epsilon_0$ with $b = 1$. Then by applying Proposition \ref{th:trace_upper_lower}, we get the lower bound
	\begin{align*}
		m = \Omega\prn*{\sqrt[2\alpha]{\frac{1}{\epsilon_0}}} = \Omega\prn*{\sqrt[2\alpha]{\frac{1}{\epsilon\abs{\alpha-1}\log n}}}.
	\end{align*}
	
	Let $Z = \tr\prn*{q_m(A)}$ be the trace of the approximated matrix functional, then
	\begin{equation*}
		\abs*{Z - \tr(A^\alpha)} = \abs*{\sum_{i=1}^n (q_m(\lambda_i) - \lambda_i^\alpha)}.
	\end{equation*}
	Let $\epsilon_0 = n^{\epsilon\abs{\alpha-1}}-1$, then by applying Proposition \ref{th:poly_lower_prop}, we have that $m = \Omega\prn*{\sqrt{\frac{v}{u}} \log\prn*{\frac{u}{v\epsilon_0}}}$ is the lower bound to achieve $\abs*{Z - \tr(A^\alpha)} \le \epsilon_0$ with $b = 1$. Combining with Proposition \ref{th:trace_upper_lower}, the lower bound for matrix function approximation is
	\begin{align*}
		m &= \Omega\prn*{\sqrt{\frac{v}{u}} \log\prn*{\frac{u}{v\epsilon_0}}} \\
		&= \Omega\prn*{\sqrt{\frac{v}{u}} \log\prn*{\frac{u}{v\epsilon\abs{\alpha-1}\log n}}}.
	\end{align*}
\end{proof}

\section*{Supplementary Experimental Results}
\subsection{Running Time of Non-Integer Approximations}
An intuitive showcase of running time with different $s$ and $m$ combinations is shown in Figures \ref{TimeExp}, in which we can observe the linear increament in time complexity with $s$ or $m$. For comparison, the trivial eigenvalue approach takes $27$ seconds for a $5000 \times 5000$ matrix.
\begin{figure}[t]
	\small
	\centering
	\includegraphics[width=0.35\textwidth]{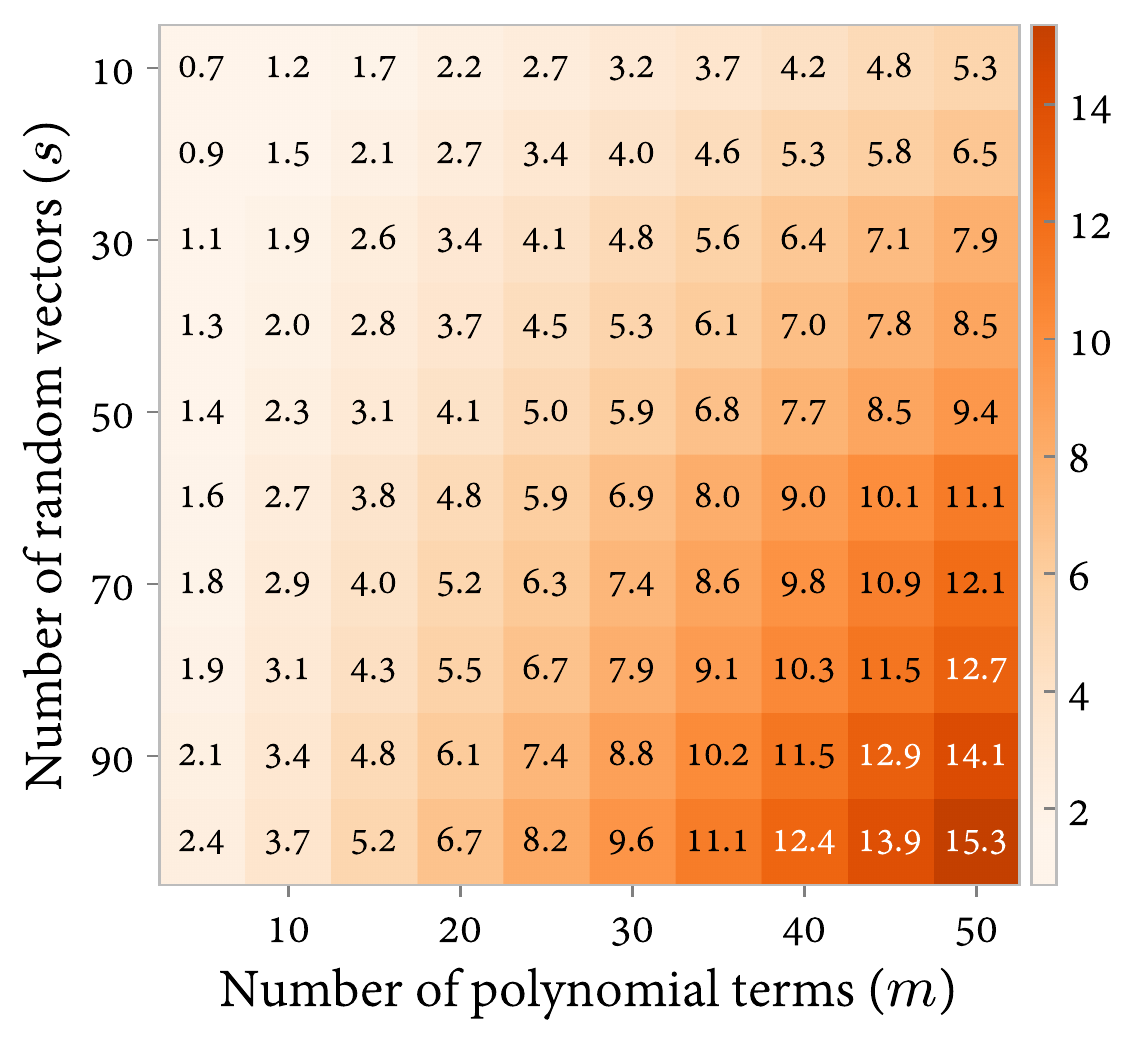}
	\caption{Running time of different $s$ and $m$ combinations.}
	\label{TimeExp}
\end{figure}

\subsection{Accuracy curves of Feature Selection}
The classification accuracy achieved by each feature selection and feature ranking methods after each incrementally selected feature are reported in Figure \ref{FeatSelFig} and \ref{FeatRankFig} respectively. It is easy to see that classification error stops to decrease after the first $10$ most informative features are selected.
\begin{figure*}[t]
	\small
	\centering
	\includegraphics[width=0.95\textwidth]{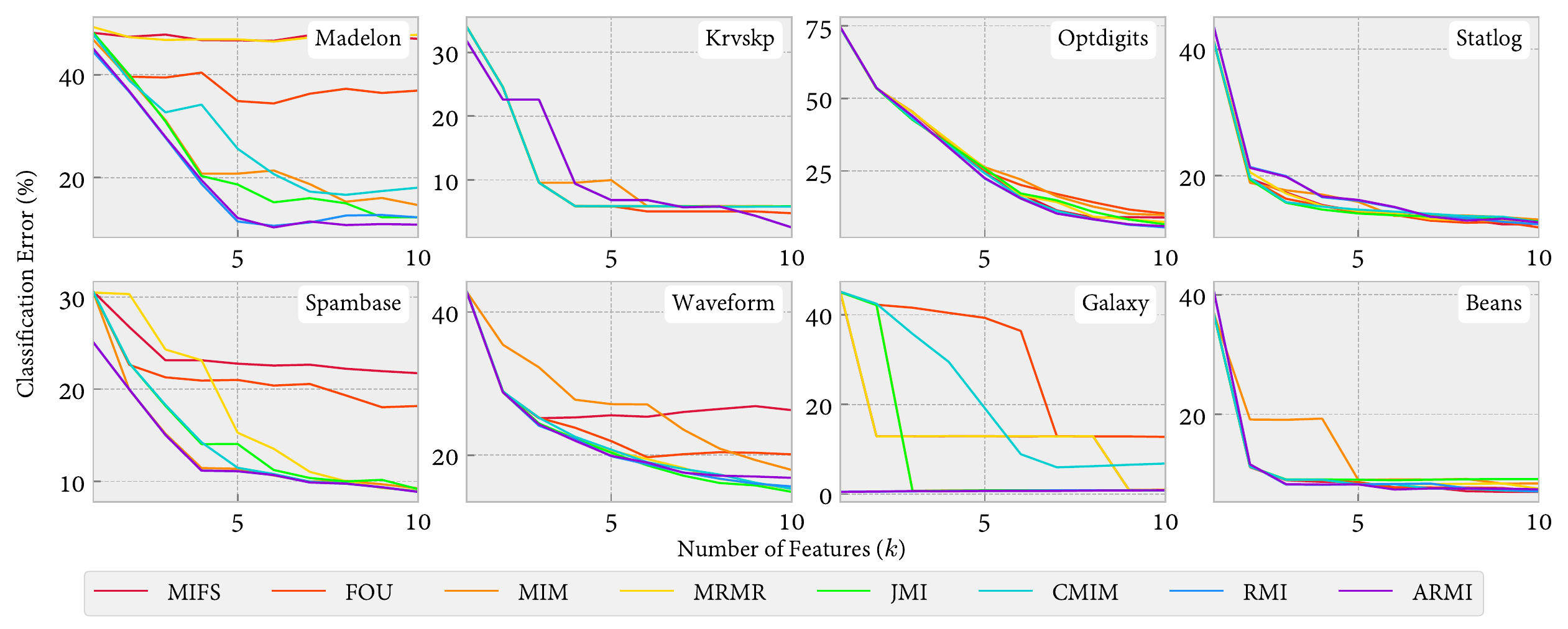}
	\caption{Number of Features ($k$) versus Classification Error (\%) curves for different feature selection methods.}
	\label{FeatSelFig}
\end{figure*}
\begin{figure*}[t]
	\small
	\centering
	\includegraphics[width=0.95\textwidth]{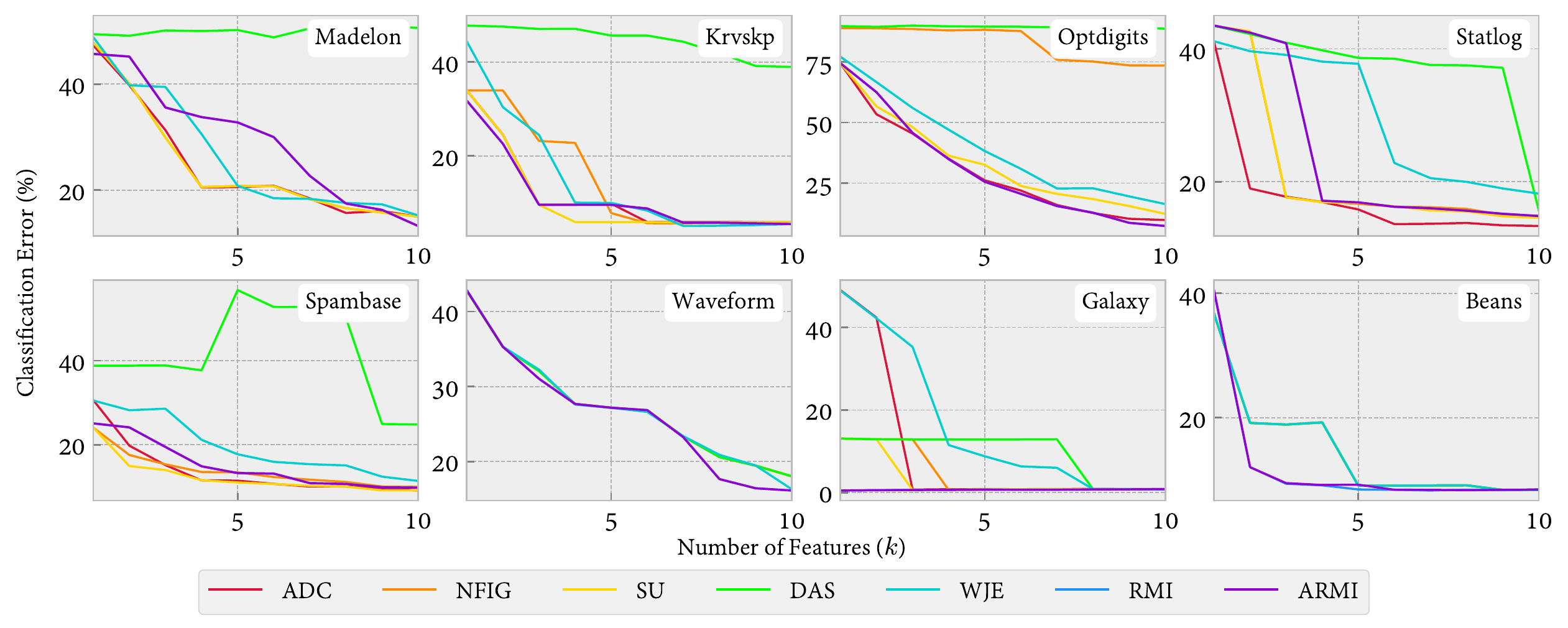}
	\caption{Number of Features ($k$) versus Classification Error (\%) curves for different feature ranking methods.}
	\label{FeatRankFig}
\end{figure*}

\bibliography{TIT_Entropy}
\bibliographystyle{IEEEtran}

\end{document}